%% file: main_neurips.tex
\documentclass{article}

\PassOptionsToPackage{numbers,sort,compress}{natbib}
\usepackage[final]{neurips_2019}

\usepackage[utf8]{inputenc} %
\usepackage[T1]{fontenc}    %
\usepackage{hyperref}       %
\usepackage{url}            %
\usepackage{booktabs}       %
\usepackage{amsfonts}       %
\usepackage{nicefrac}       %
\usepackage{microtype}      %

\usepackage{graphicx}
\usepackage{float}
\usepackage{multirow}
\usepackage{amsthm}
\usepackage{subcaption}  %
\usepackage{algorithm}
\usepackage{algpseudocode}
\newtheorem{theorem}{Theorem}
\newtheorem{lemma}{Lemma}
\theoremstyle{definition}
\newtheorem{definition}{Definition}
\input{math_commands.tex}

\DeclareMathOperator{\sigmoidfn}{sigmoid}

\title{Off-Policy Evaluation via Off-Policy Classification}

\newcommand{\google}{Google Brain, Mountain View, USA}
\newcommand{\berk}{University of California Berkeley, Berkeley, USA}
\newcommand{\deepmind}{DeepMind, London, UK}
\newcommand{\googlesym}{1}
\newcommand{\deepmindsym}{2}
\newcommand{\berksym}{3}
\newcommand{\good}{feasible}
\newcommand{\bad}{catastrophic}
\author{%
  \textbf{Alex Irpan}\textsuperscript{\textnormal{\googlesym{}}},
  \textbf{Kanishka Rao}\textsuperscript{\textnormal{\googlesym{}}},
  \textbf{Konstantinos Bousmalis}\textsuperscript{\textnormal{\deepmindsym{}}}, \\
  \textbf{Chris Harris}\textsuperscript{\textnormal{\googlesym{}}},
  \textbf{Julian Ibarz}\textsuperscript{\textnormal{\googlesym{}}},
  \textbf{Sergey Levine}\textsuperscript{\textnormal{\googlesym{},\berksym{}}} \\
  \AND
  \textnormal{\textsuperscript{\googlesym{}}}\textnormal{\google{}} \\
  \textsuperscript{\deepmindsym{}}\textnormal{\deepmind{}} \\
  \textsuperscript{\berksym{}}\textnormal{\berk{}} \\
  \AND
  \texttt{\{alexirpan,kanishkarao,konstantinos,ckharris,julianibarz,slevine\}@google.com}
}

\begin{document}

\maketitle

\begin{abstract}
In this work, we consider the problem of model selection for deep reinforcement learning (RL) in real-world environments. Typically, the performance of deep RL algorithms is evaluated via on-policy interactions with the target environment. However, comparing models in a real-world environment for the purposes of early stopping or hyperparameter tuning is costly and often practically infeasible.
This leads us to examine off-policy policy evaluation (OPE) in such settings.
We focus on OPE for value-based methods, which are of particular interest in deep RL, with applications like robotics, where off-policy algorithms based on Q-function estimation can often attain better sample complexity than direct policy optimization.
Existing OPE metrics either rely on a model of the environment, or the use of importance sampling (IS) to correct for the data being off-policy. 
However, for high-dimensional observations, such as images, models of the environment can be difficult to fit and value-based methods can make IS hard to use or even ill-conditioned, especially when dealing with continuous action spaces.
In this paper, we focus on the specific case of MDPs with continuous action spaces and sparse binary rewards, which is representative of many important real-world applications. We propose an alternative metric that relies on neither models nor IS, by framing OPE as a positive-unlabeled (PU) classification problem with the Q-function as the decision function. We experimentally show that this metric outperforms baselines on a number of tasks. Most importantly, it can reliably predict the relative performance of different policies in a number of generalization scenarios, including the transfer to the real-world of policies trained in simulation for an image-based robotic manipulation task.
\end{abstract}

\input{A-introduction.tex}

\input{B-preliminaries.tex}

\input{D-methods.tex}

\input{"C1-Generalization.tex"}

\input{"C2-related-ope.tex"}

\input{E-experiments.tex}

\input{F-conclusion.tex}

\subsubsection*{Acknowledgements}
We would like to thank Razvan Pascanu, Dale Schuurmans, George Tucker, and Paul Wohlhart for valuable discussions.

\bibliography{iclr2019_conference}
\bibliographystyle{icml2019}
\newpage
\appendix

\input{appendix_neurips.tex}

\end{document}

%% file: math_commands.tex
\newcommand{\bs}{{\bf s}}
\newcommand{\ba}{{\bf a}}
\newcommand{\br}{{\bf r}}

\newcommand{\EX}[2]{\mathbb{E}_{#1}\left[#2\right]}
\newcommand{\METRICNAME}{OPC}
\newcommand{\QDIFFNAME}{SoftOPC}
\newcommand{\code}[1]{\texttt{#1}}

\usepackage{amsmath,amsfonts,bm}

\def\eqref#1{equation~\ref{#1}}

\def\1{\bm{1}}

\DeclareMathAlphabet{\mathsfit}{\encodingdefault}{\sfdefault}{m}{sl}
\SetMathAlphabet{\mathsfit}{bold}{\encodingdefault}{\sfdefault}{bx}{n}

\newcommand{\E}{\mathbb{E}}

\DeclareMathOperator*{\argmax}{arg\,max}

\DeclareMathOperator{\sign}{sign}

%% file: A-introduction.tex
\section{Introduction}
\label{introduction}
Supervised learning has seen significant advances
in recent years, in part due to the use of large, standardized datasets~\cite{imagenet_cvpr09}. When researchers can evaluate real performance of their methods on the same data via a standardized offline metric, the progress of the field can be rapid.
Unfortunately, such metrics have been lacking in reinforcement learning (RL).
Model selection and performance evaluation in RL are typically done by estimating the average on-policy return of a method in the target environment. Although this is possible in most simulated environments~\citep{todorov2012mujoco,bellemare2013arcade,brockman2016openai}, real-world environments, like in robotics, make this difficult and expensive~\citep{Thomas2015-highconfidence}. Off-policy evaluation (OPE) has the potential to change that: a robust off-policy metric could be used together with realistic and complex data to evaluate the expected performance of off-policy RL methods, which would enable rapid progress on important real-world RL problems. Furthermore, it would greatly simplify transfer learning in RL, where OPE would enable model selection and algorithm design in simple domains (e.g., simulation) while evaluating the performance of these models and algorithms on complex domains (e.g., using previously collected real-world data).

Previous approaches to off-policy evaluation~\citep{precup2000eligibility,Dudik2011-doublyrobust,Jiang2015-doublyrobust, Thomas2016-zn} generally use importance sampling (IS) or learned dynamics models. However, this makes them difficult to use with many modern deep RL algorithms. First, OPE is most useful in the off-policy RL setting, where we expect to use real-world data as the ``validation set'', but many of the most commonly used off-policy RL methods are based on value function estimation, produce deterministic policies~\citep{lillicrap2015continuous,van2016deep}, and do not require any knowledge of the policy that generated the real-world training data. This makes them difficult to use with IS. Furthermore, many of these methods might be used with high-dimensional observations, such as images. Although there has been considerable progress in predicting future images~\citep{babaeizadeh2018stochastic, lee2018stochastic}, learning sufficiently accurate models in image space for effective evaluation is still an open research problem. We therefore aim to develop an OPE method that requires neither IS nor models.

We observe that for model selection, it is sufficient to predict some statistic \textit{correlated} with policy return, rather than directly predict policy return.
We address the specific case of binary-reward MDPs: tasks where the agent receives a non-zero reward only once during an episode%
, at the final timestep (Sect.~\ref{sec:preliminaries}).
These can be interpreted as tasks where the agent can either ``succeed'' or ``fail'' in each trial, and although they form a subset of all possible MDPs, this subset is quite representative of many real-world tasks, and is actively used e.g. in robotic manipulation~\citep{kalashnikov2018qt,riedmiller2018learning}.
The novel contribution of our method (Sect.~\ref{sec:methods}) is to frame OPE as a positive-unlabeled (PU) classification~\cite{kiryo2017positive} problem, which provides for a way to derive OPE metrics that are both \textsl{(a)} fundamentally different from prior methods based on IS and model learning, and \textsl{(b)} perform well in practice on both simulated and real-world tasks.
Additionally, we identify and present (Sect.~\ref{sec:generalization}) a list of generalization scenarios in RL that we would want our metrics to be robust against. We experimentally show (Sect.~\ref{sec:experiments}) that our suggested OPE metrics outperform a variety of baseline methods across all of the evaluation scenarios, including a simulation-to-reality transfer scenario for a vision-based robotic grasping task (see Fig.~\ref{fig:sim2real}).

\vspace{-0.1in}
\begin{figure}[t]
    \centering
    \begin{subfigure}[b]{0.42\linewidth}
        \includegraphics[width=\linewidth]{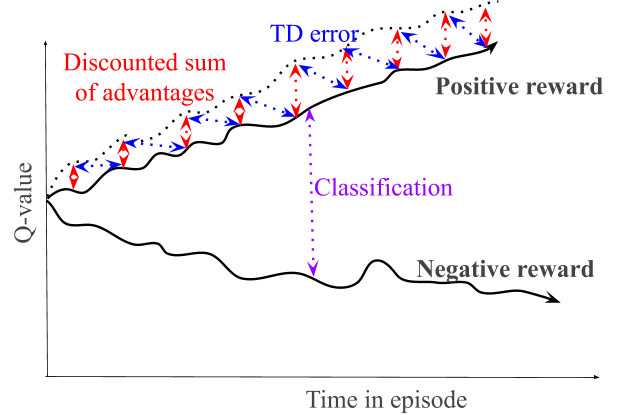}
        \caption{Visual summary of off-policy metrics}
        \label{fig:metrics_picture}
    \end{subfigure}
    \begin{subfigure}[b]{0.57\linewidth}
        \includegraphics[width=0.48\linewidth]{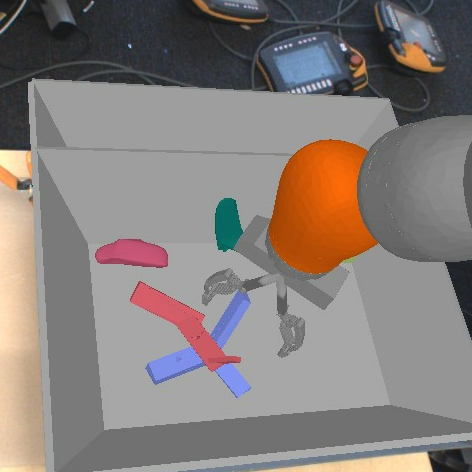}
        \includegraphics[width=0.48\linewidth]{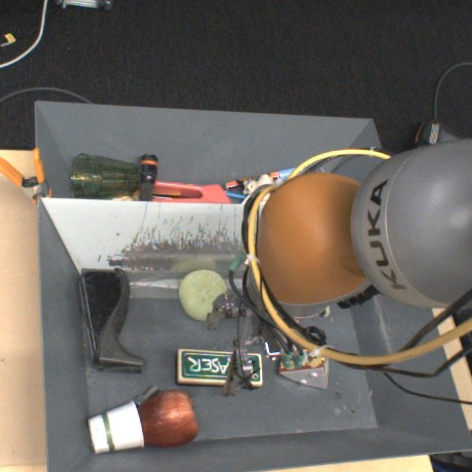}
        \caption{Robotics grasping simulation-to-reality gap.}
        \label{fig:sim2real}
    \end{subfigure}
    \caption{\textbf{(a) Visual illustration of our method:} We propose using classification-based approaches to do off-policy evaluation. Solid curves represent $Q(\bs,\ba)$ over a positive and negative trajectory, with the dashed curve representing $\max_{\ba'} Q(\bs,\ba')$ along states the positive trajectory visits (the corresponding negative curve is omitted for simplicity). Baseline approaches (blue, red) measure Q-function fit between $Q(\bs,\ba)$ to $\max_{\ba'} Q(\bs,\ba')$. Our approach (purple) directly measures separation of $Q(\bs,\ba)$ between positive and negative trajectories. \textbf{(b) The visual ``reality gap'' of our most challenging task:} off-policy evaluation of the generalization of image-based robotic agents trained solely in simulation (left) using historical data from the target real-world environment (right).}
    
\end{figure}

%% file: B-preliminaries.tex
\section{Preliminaries}
\label{sec:preliminaries}
We focus on finite--horizon Markov decision processes (MDP). We define an MDP as (${\cal S}, {\cal A}, {\cal P}, {\cal S}_0,r, \gamma$). $\cal S$ is the state--space, $\cal A$ the action--space, and both can be continuous. ${\cal P}$ defines transitions to next states given the current state and action, ${\cal S}_0$ defines initial state distribution, $r$ is the reward function, and $\gamma \in [0,1]$ is the discount factor. 
Episodes are of finite length $T$:
at a given time-step $t$ the agent is at state $\bs_t\in {\cal S}$, samples an action $\ba_t\in {\cal A}$ from a policy $\pi$, receives a reward $r_t = r(\bs_t, \ba_t)$, and observes the next state $\bs_{t+1}$ as determined by $\mathcal{P}$.

The goal in RL is to learn a policy $\pi(\ba_t|\bs_t)$ that maximizes the expected episode return $R(\pi) =  \E_{\pi}[\sum_{t=0}^{T}\gamma^tr(\bs_t,\ba_t)]$. A value of a policy for a given state $\bs_t$ is defined as $V^\pi(\bs_t)=\E_\pi{[\sum_{t'=t}^T\gamma^{t'-t}r(\bs_{t'},\ba_{t'}^\pi)]}$ where $\ba_{t}^\pi$ is the action $\pi$ takes at state $\bs_t$ and $\E_\pi$ implies an expectation over trajectories $\tau=(\bs_1, \ba_1, \ldots, \bs_T, \ba_T)$ sampled from $\pi$. Given a policy $\pi$, the expected value of its action $a_{t}$ at a state $s_{t}$ is called the Q-value and is defined as $Q^\pi(\bs_t, \ba_t)=\E_\pi{[r(\bs_t,\ba_t)+V^\pi(\bs_{t+1})]}$.

We assume the MDP is a binary reward MDP, which satisfies the following properties: $\gamma = 1$, the reward is $r_t = 0$ at all intermediate steps, and the final reward $r_T$ is in $\{0,1\}$, indicating whether the final state is a failure or a success. We learn Q-functions ${Q(\bs,\ba)}$ and aim to evaluate policies ${\pi(\bs) = \argmax_\ba Q(\bs, \ba)}$.

\subsection{Positive-unlabeled learning}
\label{sec:pu-learning}
Positive-unlabeled (PU) learning is a set of techniques for learning binary classification from partially labeled data, where we have many unlabeled points and some positively labeled points~\citep{kiryo2017positive}. We will make use of these ideas in developing our OPE metric. Positive-unlabeled data is sufficient to learn a binary classifier if the positive class prior $p(y=1)$ is known.

Let $(X, Y)$ be a labeled binary classification problem, where $Y = \{0,1\}$. Let $g:X \to \mathbb{R}$ be some decision function, and let ${\ell: \mathbb{R} \times \{0, 1\} \to \mathbb{R}}$ be our loss function. Suppose we want to evaluate loss $\ell(g(x), y)$ over negative examples $(x,y=0)$, but we only have unlabeled points $x$ and positively labeled points $(x,y=1)$. The key insight of PU learning is that the loss over negatives can be indirectly estimated from $p(y=1)$. For any $x \in X$,
\begin{equation}
    p(x) = p(x|y=1)p(y=1) + p(x|y=0)p(y=0)
\end{equation}
It follows that for any $f(x)$, \mbox{${\EX{X,Y}{f(x)} = p(y=1)\EX{X|Y=1}{f(x)}} + p(y=0)\EX{X|Y=0}{f(x)}$}, since by definition $\EX{X}{f(x)}=\int_x{p(x)f(x)dx}$. Letting $f(x) = \ell(g(x), 0)$ and rearranging gives
\begin{equation}
    p(y=0)\EX{X|Y=0}{\ell(g(x),0)} = \EX{X,Y}{\ell(g(x),0)} - p(y=1)\EX{X|Y=1}{\ell(g(x),0)}
    \label{eqn:pu-negative-risk}
\end{equation}
In Sect.~\ref{sec:methods}, we reduce off-policy evaluation of a policy $\pi$ to a positive-unlabeled classification problem. We provide reasoning for how to estimate $p(y=1)$, apply PU learning to estimate classification error with Eqn.~\ref{eqn:pu-negative-risk}, then use the error to estimate a lower bound on return $R(\pi)$ with Theorem~\ref{thm:error-bound}.

%% file: D-methods.tex
\section{Off-policy evaluation via state-action pair classification}
\label{sec:methods}

A Q-function $Q(\bs,\ba)$ predicts the expected return of each action $\ba$ given state $\bs$. The policy $\pi(\bs) = \argmax_\ba Q(\bs,\ba)$ can be viewed as a classifier that predicts the best action.
We propose an off-policy evaluation method connecting off-policy evaluation to estimating validation error for a positive-unlabeled (PU) classification problem \citep{kiryo2017positive}. Our metric can be used with Q-function estimation methods without requiring importance sampling, and can be readily applied in a scalable way to image-based deep RL tasks.

We present an analysis for binary reward MDPs, defined in Sec.~\ref{sec:preliminaries}.
In binary reward MDPs, each $(\bs_t,\ba_t)$ is either potentially effective, or guaranteed to lead to failure.

\begin{definition}
\label{defn:good-bad}
    In a binary reward MDP, $(\bs_t,\ba_t)$ is \textbf{\good{}} if an optimal policy $\pi^*$ has non-zero probability of achieving success, i.e an episode return of 1, after taking $\ba_t$ in $\bs_t$. A state-action pair $(\bs_t,\ba_t)$ is \textbf{\bad{}} if even an optimal $\pi^*$ has zero probability of succeeding if $\ba_t$ is taken. A state $\bs_t$ is \good{} if there exists a \good{} $(\bs_t, \ba_t)$, and a state $\bs_t$ is \bad{} if for all actions $\ba_t$, $(\bs_t, \ba_t)$ is \bad{}.
\end{definition}

Under this definition, the return of a trajectory $\tau$ is 1 only if all ($\bs_t, \ba_t$) in $\tau$ are \good{} (see Appendix~\ref{app:all-good}). This condition is necessary, but not sufficient, because success can depend on the stochastic dynamics. Since Definition~\ref{defn:good-bad} is defined by an optimal $\pi^*$, we can view \good{} and \bad{} as binary labels that are independent of the policy $\pi$ we are evaluating. Viewing $\pi$ as a classifier, we relate
the classification error of $\pi$ to a lower bound for return $R(\pi)$.

\begin{theorem}
\label{thm:error-bound}
Given a binary reward MDP and a policy $\pi$, let $c(\bs_t,\ba_t)$ be the probability that stochastic dynamics bring a \good{} $(\bs_t,\ba_t)$ to a \bad{} $\bs_{t+1}$, with $c = \max_{\bs,\ba} c(\bs,\ba)$. Let $\rho^{+}_{t,\pi}$ denote the state distribution at time $t$, given that $\pi$ was followed, all its previous actions $\ba_1, \cdots, \ba_{t-1}$ were \good{}, and $\bs_t$ is \good{}. Let $\mathcal{A}_-(\bs)$ denote the set of \bad{} actions at state $\bs$,
and let $\epsilon_t = \EX{\rho^{+}_{t,\pi}}{\sum_{\ba \in \mathcal{A}_-(\bs_t)} \pi(\ba|\bs_t)}$ be the per-step expectation of $\pi$ making its first mistake at time $t$, with $\epsilon = \frac{1}{T} \sum_{i=1}^T \epsilon_t$ being average error over all $(\bs_t,\ba_t)$.
Then $R(\pi) \ge 1 - T(\epsilon + c)$.
\end{theorem}

See Appendix~\ref{app:error-bound-proof} for the proof. For the deterministic case ($c = 0$), we can take inspiration from 
imitation learning behavioral cloning bounds in~\citet{ross2010efficient} to prove the same result. This alternate proof is in Appendix~\ref{app:behavior-clone}.

A smaller error $\epsilon$ gives a higher lower bound on return, which implies a better $\pi$. This leaves estimating $\epsilon$.
The primary challenge with this approach is that we do not have negative labels -- that is, for trials that receive a return of 0 in the validation set, we do not know which $(\bs,\ba)$ were in fact \bad{}, and which were recoverable. We discuss how we address this problem next.

\subsection{Missing negative labels}
\label{sec:pu-derivation}
Recall that $(\bs_t, \ba_t)$ is \good{} if $\pi^*$ has a chance of succeeding after action $\ba_t$.
Since $\pi^*$ is at least as good as $\pi_b$, whenever $\pi_b$ succeeds, all tuples $(\bs_t, \ba_t)$ in the trajectory $\tau$ must be \good{}. However, the converse is not true, since failure could come from poor luck or suboptimal actions. Our key insight is that this is an instance of the positive-unlabeled (PU) learning problem from Sect.~\ref{sec:pu-learning}, where $\pi_b$ positively labels some $(\bs,\ba)$ and the remaining are unlabeled. This lets us use ideas from PU learning to estimate error.

In the RL setting, inputs $(\bs,\ba)$ are from $X = \mathcal{S} \times \mathcal{A}$, labels $\{0,1\}$ correspond to $\{\bad{}, \good{}\}$ labels, and a natural choice for the decision function $g$ is $g(\bs, \ba) = Q(\bs,\ba)$, since $Q(\bs,\ba)$ should be high for \good{} $(\bs,\ba)$ and low for \bad{} $(\bs,\ba)$.
We aim to estimate $\epsilon$, the probability that $\pi$ takes a \bad{} action -- i.e., that $(\bs, \pi(\bs))$ is a false positive.
Note that if $(\bs, \pi(\bs))$ is predicted to be \bad{}, but is actually \good{}, this false-negative does not impact future reward -- since the action is \good{}, there is still some chance of success. We want just the false-positive risk, ${\epsilon = p(y=0)\EX{X|Y=0}{\ell(g(x),0)}}$. This is the same as Eqn.~\ref{eqn:pu-negative-risk}, and using $g(\bs,\ba) = Q(\bs,\ba)$ gives
\begin{equation}
\label{eqn:error-estimate}
    \epsilon = \EX{(\bs,\ba)}{\ell(Q(\bs,\ba),0)} - p(y=1) \EX{(\bs,\ba),y=1}{\ell(Q(\bs,\ba),0)}.
\end{equation}
Eqn.~\ref{eqn:error-estimate} is the core of all our proposed metrics.
While it might at first seem that the class prior $p(y=1)$ should be task-dependent, recall that the error $\epsilon_t$ is the expectation over the state distribution $\rho^+_{t,\pi}$, where the actions $\ba_1,\cdots,\ba_{t-1}$ were all \good{}. This is equivalent to following an optimal ``expert'' policy $\pi^*$, and although we are estimating $\epsilon_t$ from data generated by behavior policy $\pi_b$, we should match the positive class prior $p(y=1)$ we would observe from expert $\pi^*$.
Expert $\pi^*$ will always pick \good{} actions.  Therefore, although the validation dataset will likely have both successes and failures, a prior of $p(y=1) = 1$ is the ideal prior, and this holds independently of the environment. We illustrate this further with a didactic example in Sect.~\ref{sect:smallscale}.

Theorem~\ref{thm:error-bound} relies on estimating $\epsilon$ over the distribution $\rho^{+}_{t,\pi}$, but our dataset $\mathcal{D}$ is generated by an unknown behavior policy $\pi_b$.
A natural approach here would be importance sampling (IS)~\citep{Dudik2011-doublyrobust}, but: 
\textsl{(a)} we assume no knowledge of $\pi_b$, and \textsl{(b)} IS is not well-defined for deterministic policies \mbox{$\pi(\bs) = \argmax_\ba Q(\bs, \ba)$}. Another approach is to subsample $\mathcal{D}$ to transitions $(\bs, \ba)$ where ${\ba = \pi(\bs)}$~\citep{liu2018representation}. This ensures an on-policy evaluation, but can encounter finite sample issues if $\pi_b$ does not sample $\pi(\bs)$ frequently enough.
Therefore, we assume classification error over $\mathcal{D}$ is a good enough proxy that correlates well with classification error over $\rho^{+}_{t,\pi}$.
This is admittedly a strong assumption, but empirical results in Sect.~\ref{sec:experiments} show surprising robustness to distributional mismatch. This assumption is reasonable if $\mathcal{D}$ is broad (e.g., generated by a sufficiently random policy), but may produce pessimistic estimates when potential \good{} actions
in $\mathcal{D}$ are unlabeled.

\subsection{Off-policy classification for OPE}

Based off of the derivation from Sect.~\ref{sec:pu-derivation},
our proposed off-policy classification (\METRICNAME{}) score is defined by the negative loss when $\ell$ in Eqn.~\ref{eqn:error-estimate} is the 0-1 loss. Let $b$ be a threshold, with ${\ell(Q(\bs,\ba), Y) = \frac{1}{2} + \left(\frac{1}{2} - Y\right)\sign(Q(\bs,\ba) - b)}$. This gives
\begin{equation}
    \text{\METRICNAME{}}(Q) = p(y=1) \EX{(\bs,\ba),y=1}{1_{Q(\bs,\ba)>b}} - \EX{(\bs,\ba)}{1_{Q(\bs,\ba)>b}}.
\end{equation}
To be fair to each $Q(\bs,\ba)$, threshold $b$ is set separately for each Q-function to maximize $\text{\METRICNAME{}}(Q)$. Given $N$ transitions and $Q(\bs,\ba)$ for all $(\bs,\ba) \in \mathcal{D}$, this can be done in $O(N \log N)$ time per Q-function (see Appendix~\ref{app:efficient-compute}). This avoids favoring Q-functions that systematically overestimate or underestimate the true value.

Alternatively, $\ell$ can be a soft loss function. We experimented with $\ell(Q(\bs,\ba), Y) = (1-2Y)Q(\bs,\ba)$, which is minimized when $Q(\bs,\ba)$ is large for $Y = 1$ and small for $Y = 0$. The negative of this loss is called the \QDIFFNAME{}.
\begin{equation}
    \text{\QDIFFNAME{}}(Q) = p(y=1) \EX{(\bs,\ba),y=1}{Q(\bs,\ba)} - \EX{(\bs,\ba)}{Q(\bs,\ba)}.
\end{equation}
If episodes have different lengths, to avoid focusing on long episodes, transitions $(\bs,\ba)$ from an episode of length $T$ are weighted by $\frac{1}{T}$ when estimating \QDIFFNAME{}. Pseudocode is in Appendix~\ref{app:efficient-compute}.

Although our derivation is for binary reward MDPs, both the \METRICNAME{} and \QDIFFNAME{} are purely evaluation time metrics, and can be applied to Q-functions trained with dense rewards or reward shaping, as long as the final evaluation uses a sparse binary reward.

\subsection{Evaluating OPE metrics}
\label{sec:eval-correlation}
The standard evaluation method for OPE is to report MSE to the true episode return~\citep{Thomas2016-zn,liu2018representation}. However, our metrics
do not estimate episode return directly. The $\text{\METRICNAME{}}(Q)$'s estimate of $\epsilon$ will differ from the true value, since it is estimated over our dataset $\mathcal{D}$ instead of over the distribution $\rho^+_{t,\pi}$. Meanwhile, $\text{\QDIFFNAME{}}(Q)$ does not estimate $\epsilon$ directly due to using a soft loss function. Despite this, the \METRICNAME{} and \QDIFFNAME{} are still useful OPE metrics if they \textit{correlate} well with $\epsilon$ or episode return $R(\pi)$.

We propose an alternative evaluation method. Instead of reporting MSE, we train a large suite of Q-functions $Q(\bs,\ba)$ with different learning algorithms, evaluating true return of the equivalent argmax policy for each $Q(\bs,\ba)$, then compare correlation of the metric to true return. We report two correlations, the coefficient of determination $R^2$ of line of best fit, and the Spearman rank correlation $\xi$~\citep{spearman1904}.\footnote{We slightly abuse notation here, and should clarify that $R^2$ is used to symbolize the coefficient of determination and should not be confused with $R(\pi)$, the average return of a policy $\pi$.} $R^2$ measures confidence in how well our linear best fit will predict returns of new models, whereas $\xi$ measures confidence that the metric ranks different policies correctly, without assuming a linear best fit.

%% file: C1-Generalization.tex
\section{Applications of OPE for transfer and generalization}
\label{sec:generalization}
Off-policy evaluation (OPE)
has many applications. One is to use OPE as an early stopping or model selection criteria when training from off-policy data. Another is applying OPE to validation data collected in another domain to measure generalization to new settings.
Several papers \citep{Zhang2018-qj,raghu2018can,cobbe2018quantifying,Zhang2018-ah,nichol2018gotta} have examined overfitting and memorization in deep RL, proposing explicit train-test environment splits as benchmarks for RL generalization.
Often, these test environments are defined in simulation, where
it is easy to evaluate the policy in the test environment. This is no longer sufficient for real-world settings, where
test environment evaluation can be expensive. In real-world problems, off-policy evaluation is an inescapable part of measuring generalization performance in an efficient, tractable way.
To demonstrate this, we identify a few common generalization failure scenarios faced in reinforcement learning, applying OPE to each one.
When there is \textit{insufficient off-policy training data} and new data is not collected online, models may memorize state-action pairs in the training data.
 RL algorithms collect new on-policy data with high frequency. If training data is generated in a systematically biased way, we have \textit{mismatched off-policy training data.} The model fails to generalize because systemic biases cause the model to miss parts of the target distribution.
Finally, models trained in simulation usually do not generalize to the real-world, due to the \textit{training and test domain gap}: the differences in the input space (see Fig.~\ref{fig:sim2real} and Fig.~\ref{fig:procedural_training}) and the dynamics.
All of these scenarios are, in principle, identifiable by off-policy evaluation, as long as validation is done against data sampled from the final testing environment. We evaluate our proposed and baseline metrics across all these scenarios.

\begin{figure}[t]
\centering
\begin{subfigure}[b]{0.35\linewidth}
\includegraphics[width=\textwidth]{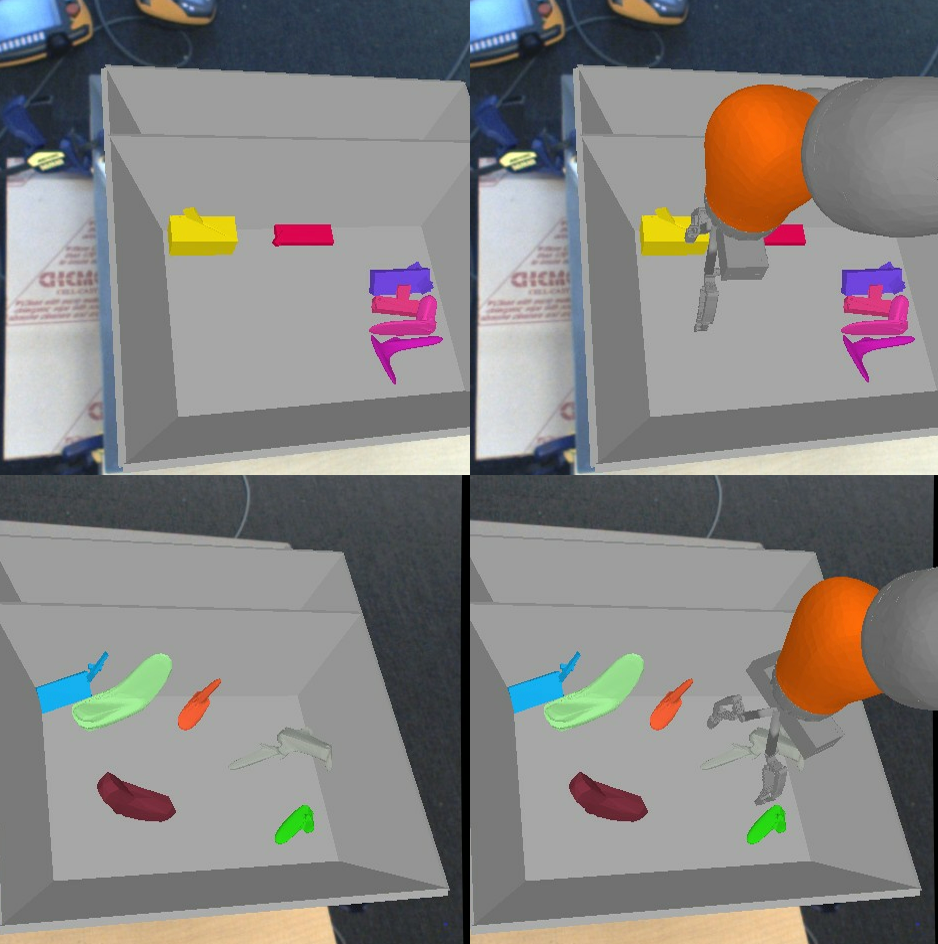}
\caption{Simulated samples}
\end{subfigure}
\begin{subfigure}[b]{0.35\linewidth}
\includegraphics[width=\textwidth]{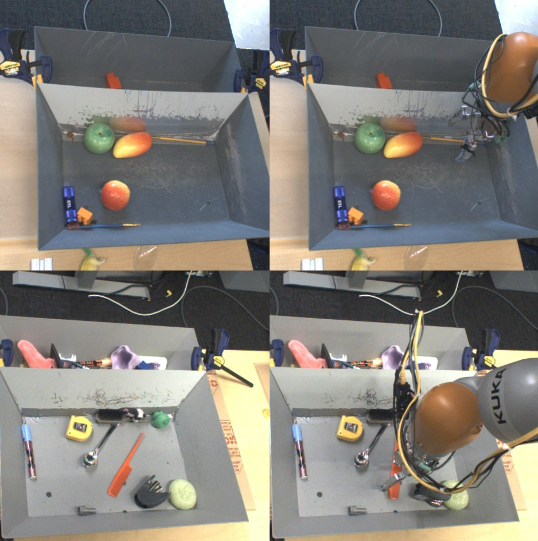}
\caption{Real samples}
\end{subfigure}
\caption{\textbf{An example of a training and test domain gap.} We display this with a robotic grasping task.
\textit{Left:} Images used during training, from (a) simulated grasping over procedurally generated objects; and from (b) the real-world, with a varied collection of everyday physical objects.\vspace{-5mm}}
\label{fig:procedural_training}
\end{figure}

%% file: C2-related-ope.tex
\section{Related work}
\label{related}

Off-policy policy evaluation (OPE) predicts the return of a learned policy $\pi$ from a fixed off-policy dataset $\mathcal{D}$, generated by one or more behavior policies $\pi_b$.
Prior works~\citep{precup2000eligibility, Dudik2011-doublyrobust, Jiang2015-doublyrobust, liu2018representation,theocharous2015personalized,Hanna2017-yr} do so with importance sampling (IS)~\citep{horvitz1952generalization}, MDP modeling, or both.
Importance sampling requires querying $\pi(\ba|\bs)$ and $\pi_b(\ba|\bs)$ for any $\bs\in\mathcal{D}$, to correct for the shift in state-action distributions. In RL, the cumulative product of IS weights along $\tau$ is used to weight its contribution to $\pi$'s estimated value~\citep{precup2000eligibility}.
Several variants have been proposed, such as step-wise IS and weighted IS~\citep{mahmood2014weighted}.
In MDP modeling, a model is fitted to $\mathcal{D}$, and $\pi$ is rolled out in the learned model to estimate average return~\citep{mannor2007bias,Jiang2015-doublyrobust}.
The performance of these approaches is worse if dynamics or reward are poorly estimated, which tends to occur for image-based tasks. Improving these models is
an active research question~\citep{babaeizadeh2018stochastic, lee2018stochastic}.
State of the art methods combine IS-based estimators and model-based estimators using doubly robust estimation and ensembles to produce improved estimators with theoretical guarantees~\citep{Dudik2011-doublyrobust,dudik2014doubly,Jiang2015-doublyrobust,Thomas2016-zn,Hanna2017-yr}.

These IS and model-based OPE approaches assume importance sampling or model learning are feasible. This assumption often breaks down in modern deep RL approaches.
When $\pi_b$ is unknown, $\pi_b(\ba|\bs)$ cannot be queried. When doing value-based RL with deterministic policies, $\pi(\ba|\bs)$ is undefined for off-policy actions. When working with high-dimensional observations, model learning is often too difficult to learn a reliable model for evaluation.

Many recent papers \citep{Zhang2018-qj,raghu2018can,cobbe2018quantifying,Zhang2018-ah,nichol2018gotta} have defined train-test environment splits to evaluate RL generalization, but define test environments in simulation where there is no need for OPE. We demonstrate how OPE provides tools to evaluate RL generalization for real-world environments.
While to our knowledge no prior work has proposed a classification-based OPE approach, several prior works have used supervised classifiers to predict transfer performance from a few runs in the test environment~\cite{koos2012transferability,koos2010crossing}.
To our knowledge, no other OPE papers have shown results for large image-based tasks where neither importance sampling nor model learning are viable options.

\paragraph{Baseline metrics}
\label{sec:baselines}
Since we assume importance-sampling and model learning are infeasible,
many common OPE baselines do not fit our problem setting. In their place, we use other Q-learning based metrics that also do not need importance sampling or model learning and %
only require a $Q(\bs,\ba)$ estimate. The \textit{temporal-difference error} (TD Error) is the standard Q-learning training loss, and \citet{Farahmand2011-ec} proposed a model selection algorithm based on minimizing TD error. The \textit{discounted sum of advantages} ($\sum_t \gamma^t A^\pi)$ relates the difference in values $V^{\pi_b}(\bs) - V^\pi(\bs)$ to the sum of advantages $\sum_t \gamma^t A^\pi(\bs,\ba)$ over data from $\pi_b$, and was proposed by \citet{kakade2002} and \citet{Murphy2005-qc}. Finally, the \textit{Monte Carlo corrected error} (MCC Error) is derived by arranging the discounted sum of advantages into a squared error, and was used as a training objective by~\citet{Quillen2018-ms}.
The exact expression of each of these metrics is in Appendix~\ref{app:baseline_impl}.

Each of these baselines represents a different way to measure how well $Q(\bs,\ba)$ fits the true return. However, it is possible to learn a good policy $\pi$ even when $Q(\bs,\ba)$ fits the data poorly. In Q-learning, it is common to define an argmax policy $\pi = \argmax_\ba Q(\bs,\ba)$.
The argmax policy for $Q^*(\bs,\ba)$ is $\pi^*$, and $Q^*$ has zero TD error. But, applying any monotonic function to $Q^*(\bs,\ba)$ produces a $Q'(\bs,\ba)$, whose TD error is non-zero, but whose argmax policy is still $\pi^*$. A good OPE metric should rate $Q^*$ and $Q'$ identically.
This motivates our proposed classification-based OPE metrics: since $\pi$'s behavior only depends on the relative differences in Q-value, it makes sense to directly contrast Q-values against each other, rather than compare error between the Q-values and episode return. Doing so lets us compare Q-functions whose $Q(\bs,\ba)$ estimates are inaccurate. Fig.~\ref{fig:metrics_picture} visualizes the differences between the baseline metrics and classification metrics.

%% file: E-experiments.tex
\section{Experiments}
\label{sec:experiments}
In this section, we investigate the correlation of \METRICNAME{} and \QDIFFNAME{} with true average return, and how they may be used for model selection with off-policy data. We compare the correlation of these metrics with the correlation of the baselines, namely the TD Error, Sum of Advantages, and the MCC Error (see Sect.~\ref{sec:baselines}) in a number of environments and generalization failure scenarios.
For each experiment, a validation dataset $\mathcal{D}$ is collected with a behavior policy $\pi_{b}$, and state-action pairs $(\bs,\ba)$ are labeled as \good{} whenever they appear in a successful trajectory.
In line with Sect.~\ref{sec:eval-correlation}, several Q-functions $Q(\bs,\ba)$ are trained for each task.
For each $Q(\bs,\ba)$, we evaluate each metric over $\mathcal{D}$
and true return of the equivalent argmax policy. We report both the coefficient of determination $R^2$ of line of best fit
and the Spearman's rank correlation coefficient $\xi$~\citep{spearman1904}. Our results are summarized in Table~\ref{tab:random_results} and Table~\ref{tab:full_results}.
Our OPC/SoftOPC metrics are implemented using $p(y=1) = 1$, as explained in Sect.~\ref{sec:methods} and Appendix~\ref{app:prior}.

\subsection{Simple Environments}
\label{sect:smallscale}

\paragraph{Binary tree.} 
\label{sect:tree}
As a didactic toy example, we used a binary tree MDP with depth of episode length $T$. In this environment,\footnote{Code for the binary tree environment is available at \url{https://bit.ly/2Qx6TJ7}.} each node is a state $\bs_t$ with $r_t=0$, unless it is a leaf/terminal state with reward $r_T\in \{0,1\}$. Actions are $\{\mbox{`left'},\mbox{`right'}\}$, and transitions are deterministic. Exactly one leaf is a success leaf with $r_T = 1$, and the rest have $r_T = 0$.
In our experiments we used a full binary tree of depth $T=6$. The initial state distribution was uniform over all non-leaf nodes, which means that the initial state could sometimes be initialized to one where failure is inevitable. The validation dataset $\mathcal{D}$ was collected by generating 1,000 episodes from a uniformly random policy. For the policies we wanted to evaluate, we generated 1,000 random Q-functions by sampling $Q(\bs, \ba) \sim U[0,1]$ for every $(\bs, \ba)$, defining the policy as $\pi(\bs) = \argmax_\ba Q(\bs,\ba)$. We compared the correlation of the actual on-policy performance of the policies with the scores given by the OPC, SoftOPC and the baseline metrics using $\mathcal{D}$, as shown in Table~\ref{tab:full_results}. \QDIFFNAME{} correlates best and \METRICNAME{} correlates second best.

\paragraph{Pong.}
As we are specifically motivated by image-based tasks with binary rewards, the Atari~\cite{bellemare2013arcade} Pong game was a good choice for a simple environment that can have these characteristics. The visual input is of low complexity, and the game can be easily converted into a binary reward task by truncating the episode after the first point is scored.
We learned Q-functions using DQN~\citep{mnih2015human} and DDQN~\citep{van2016deep}, varying hyperparameters such as the learning rate, the discount factor $\gamma$, and the batch size, as discussed in detail in Appendix~\ref{app:pong_details}. A total of 175 model checkpoints are chosen from the various models for evaluation, and true average performance is evaluated over 3,000 episodes for each model checkpoint.
For the validation dataset we used 38 Q-functions that were partially-trained with DDQN and generated 30 episodes from each, for a total of 1140 episodes. Similarly with the Binary Tree environments we compare the correlations of our metrics and the baselines to the true average performance over a number of on-policy episodes.
As we show in Table~\ref{tab:full_results}, both our metrics outperform the baselines,  \METRICNAME{} performs better than \QDIFFNAME{} in terms of $R^2$ correlation but is similar in terms of Spearman correlation $\xi$.

\paragraph{Stochastic dynamics.} To evaluate performance against stochastic dynamics, we modified the dynamics of the binary tree and Pong environment. In the binary tree, the environment executes a random action instead of the policy's action with probability $\epsilon$. In Pong, the environment uses sticky actions, a standard protocol for stochastic dynamics in Atari games introduced by~\cite{machado2018revisiting}. With small probability, the environment repeats the previous action instead of the policy's action. Everything else is unchanged. Results in Table~\ref{tab:random_results}. In more stochastic environments, all metrics drop in performance since $Q(s,a)$ has less control over return, but OPC and SoftOPC consistently correlate better than the baselines.

\begin{table*}[th]
\caption{Results from stochastic dynamics experiments. For each metric (leftmost column), we report $R^2$ of line of best fit and Spearman rank correlation coefficient $\xi$ for each environment (top row), over stochastic versions of the binary tree and Pong tasks from Sect.~\ref{sect:smallscale}. Correlation drops as stochasticity increases, but our proposed metrics (last two rows) consistently outperform baselines.}
\centering
\begin{small}
\begin{tabular}{l|cc|cc|cc|cc|cc}
    & \multicolumn{6}{c|}{\textbf{Stochastic Tree 1-Success Leaf}} &
    \multicolumn{4}{c}{\textbf{Pong Sticky Actions}} \\
    &
    \multicolumn{2}{c}{$\bm{\epsilon=0.4}$} & \multicolumn{2}{c}{$\bm{\epsilon=0.6}$} & \multicolumn{2}{c|}{$\bm{\epsilon=0.8}$} &
    \multicolumn{2}{c}{\textbf{Sticky 10\%}} &
    \multicolumn{2}{c}{\textbf{Sticky 25\%}}
    \\ \hline
        & $R^2$ & $\xi$ &
        $R^2$ & $\xi$ &
        $R^2$ & $\xi$ &
        $R^2$ & $\xi$ &
        $R^2$ & $\xi$
        \\ \hline
        \textbf{TD Err} &
        0.01 & -0.07 & 0.00 & -0.05 & 0.00 & -0.05 & 0.05 & -0.16 & 0.07 & -0.15 \\
        $\sum \bm{\gamma^{t}A^\pi}$ &
        0.00 & 0.01 & 0.01 & -0.07 & 0.00 & -0.02 & 0.04 & -0.29 & 0.01 & -0.22 \\
        \textbf{MCC Err} &
        0.07 & -0.27 & 0.01 & -0.06 & 0.01 & -0.11 & 0.02 & -0.32 & 0.00 & -0.18 \\ \hline
        \textbf{\METRICNAME{} (Ours)} &
        0.13 & 0.38 & 0.01 & 0.08 & 0.03 & 0.19 & \textbf{0.48} & \textbf{0.73} & \textbf{0.33} & \textbf{0.66} \\
        \textbf{\QDIFFNAME{} (Ours)} &
        \textbf{0.14} & \textbf{0.39} & \textbf{0.03} & \textbf{0.18} & \textbf{0.04} & \textbf{0.20} & 0.33 & 0.67 & 0.16 & 0.58
\end{tabular}
\end{small}
\label{tab:random_results}
\end{table*}
\vspace{-0.15in}

\subsection{Vision-based Robotic Grasping}

Our main experimental results were on simulated and real versions of a  robotic environment and a vision-based grasping task, following the setup from~\citet{kalashnikov2018qt}, the details of which we briefly summarize. The observation at each time-step is a $472\times472$ RGB image
from a camera placed over the shoulder of a robotic arm,
of the robot and a bin of objects, as shown in Fig.~\ref{fig:sim2real}.
At the start of an episode, objects are randomly dropped in a bin in front of the robot.
The goal is to grasp any of the objects in that bin.
Actions include continuous Cartesian displacements of the gripper, and the rotation of the gripper around the z-axis. The action space also includes three discrete commands: ``open gripper'', ``close gripper'', and ``terminate episode''. Rewards are sparse, with $r(\bs_T, \ba_T) = 1$ if any object is grasped
and $0$ otherwise. All models are trained with 
the fully off-policy QT-Opt algorithm as described in ~\citet{kalashnikov2018qt}. 

In simulation we define a training and a test environment by generating two distinct sets of 5 objects that are used for each, shown in Fig.~\ref{fig:train_test_objects}. In order to capture the different possible generalization failure scenarios discussed in Sect.~\ref{sec:generalization}, we trained Q-functions in a fully off-policy fashion with data collected by a hand-crafted policy with a 60\% grasp success rate and $\epsilon$-greedy exploration (with $\epsilon$=0.1) with two different datasets both from the training environment. The first consists of $100,000$ episodes, with which we can show we have \textit{insufficient off-policy training data} to perform well even in the training environment. The second consists of $900,000$ episodes, with which we can show we have sufficient data to perform well in the training environment, but due to \textit{mismatched off-policy training data} we can show that the policies do not generalize to the test environment (see Fig.~\ref{fig:train_test_objects} for objects and Appendix~\ref{app:sim_details} for the analysis). We saved policies at different stages of training which resulted in 452 policies for the former case and 391 for the latter. We evaluated the true return of these policies on 700 episodes on each environment and calculated the correlation with the scores assigned by the OPE metrics based on held-out validation sets of $50,000$ episodes from the training environment and $10,000$ episodes from the test one, which we show in  Table~\ref{tab:full_results}.

\begin{table*}[th]
    \caption{Summarized results of Experiments section. For each metric (leftmost column), we report $R^2$ of line of best fit and Spearman rank correlation coefficient $\xi$ for each environment (top row). These are: the binary tree and Pong tasks from Sect.~\ref{sect:smallscale}, simulated grasping with train or test objects, and real-world grasping from Sect.~\ref{sec:sim-robot}. Baseline metrics are discussed in Sect.~\ref{related}, and our metrics (\METRICNAME{}, \QDIFFNAME{}) are discussed in Sect.~\ref{sec:methods}. Occasionally, some baselines correlate well, but our proposed metrics (last two rows) are consistently among the best metrics for each environment.}
    \centering
    \vspace{-0.1in}
    \begin{small}
    \begin{tabular}{l|cc|cc|cc|cc|cc}
        &
        \multicolumn{2}{c|}{\textbf{Tree (1 Succ)}} & \multicolumn{2}{c|}{\textbf{Pong}} & \multicolumn{2}{c|}{\textbf{Sim Train}} & \multicolumn{2}{c|}{\textbf{Sim Test}} & \multicolumn{2}{c}{\textbf{Real-World}}  \\ \hline
        & $R^2$ & $\xi$ &
        $R^2$ & $\xi$ &
        $R^2$ & $\xi$ &
        $R^2$ & $\xi$ &
        $R^2$ & $\xi$ \\ \hline
        \textbf{TD Err} &
        0.02 & -0.15 & 0.05 & -0.18 & 0.02 & -0.37 & 0.10 & -0.51 & 0.17 & 0.48 \\
        $\sum \bm{\gamma^{t}A^\pi}$ &
        0.00 & 0.00 & 0.09 & -0.32 & \textbf{0.74} & 0.81 & \textbf{0.74} & \textbf{0.78} & 0.12 & 0.50 \\
        \textbf{MCC Err} &
        0.06 & -0.26 & 0.04 & -0.36 & 0.00 & 0.33 & 0.06 & -0.44 & 0.01 & -0.15 \\ \hline
        \textbf{\METRICNAME{} (Ours)} &
        \textbf{0.21} & 0.50 & \textbf{0.50} & 0.72 & 0.49 & \textbf{0.86} & 0.35 & 0.66 & 0.81 & 0.87 \\
        \textbf{\QDIFFNAME{} (Ours)} &
        0.19 & \textbf{0.51} & 0.36 & \textbf{0.75} & 0.55 & 0.76 & 0.48 & 0.77 & \textbf{0.91} & \textbf{0.94} \\
    \end{tabular}
    \end{small}
   \label{tab:full_results}
\end{table*}

The real-world version of the environment has objects that were never seen during training (see  Fig.~\ref{fig:sim2real} and~\ref{fig:kuka_real_v_sim}). We evaluated 15 different models, trained to have varying degrees of robustness to the \textit{training and test domain gap}, based on domain randomization and randomized--to-canonical adaptation networks~\citep{james2019rcan}.\footnote{For full details of each of the models please see Appendix~\ref{app:real_details}.} Out of these, 7  were trained on-policy purely in simulation. True average return was evaluated over 714 episodes with 7 different sets of objects, and true policy real-world performance ranged from 17\% to 91\%.
The validation dataset consisted of $4,000$ real-world episodes, 40\% of which were successful grasps and the objects used for it were separate from the ones used for final evaluation used for the results in Table~\ref{tab:full_results}.

\label{sec:sim-robot}

\vspace{-0.1in}
\begin{figure}[h]
\centering
    \begin{subfigure}[b]{0.35\linewidth}  %
        \includegraphics[width=\linewidth]{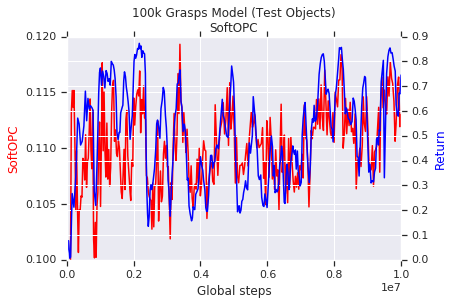}  %
        \vspace{-0.21in} %
        \caption{\QDIFFNAME{} and return in sim}
        \label{fig:correlation_scatter}
    \end{subfigure}
    \begin{subfigure}[b]{0.64\linewidth}  %
        \includegraphics[width=0.48\linewidth]{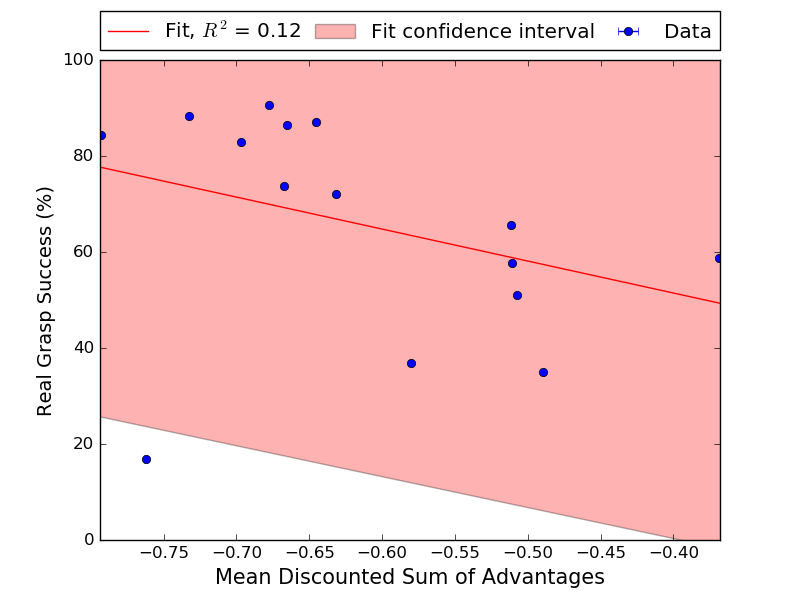}
        \includegraphics[width=0.48\linewidth]{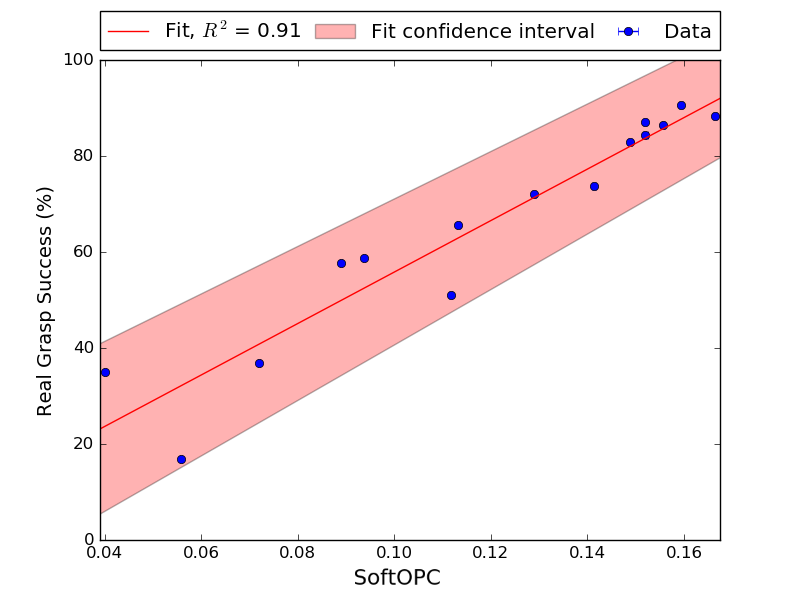}
        \vspace{-0.1in}
        \caption{Scatterplot for real-world grasping}
        \label{fig:all_metrics_all_policies}
    \end{subfigure}
    \caption{
    \textbf{(a): \QDIFFNAME{} in simulated grasping.}
    Overlay of \QDIFFNAME{} (red) and return (blue) in simulation for model trained with 100k grasps. \QDIFFNAME{} tracks episode return.
    \textbf{(b): Scatterplots for OPE metrics and real-world grasp success.} 
        Scatterplots for $\sum \gamma^{t'} A^\pi(\bs_{t'}, \ba_{t'})$ (left)  and \QDIFFNAME{} (right) for the Real-World grasping task. Each point is a different grasping model.
        Shaded regions are a 95\% confidence interval.
        $\sum \gamma^{t'} A^\pi(\bs_{t'}, \ba_{t'})$ works in simulation but fails on real data, whereas \QDIFFNAME{} works well in both.}
\end{figure}

\vspace{-0.1in}

\vspace{-0.1in}

\subsection{Discussion}
Table~\ref{tab:full_results} shows $R^2$ and $\xi$ for each metric for the different environments we considered. Our proposed \QDIFFNAME{} and \METRICNAME{} consistently outperformed the baselines, with the exception of the simulated robotic test environment, on which the SoftOPC performed almost as well as the discounted sum of advantages on the Spearman correlation (but worse on $R^2$). 
However, we show that SoftOPC more reliably ranks policies than the baselines for real-world performance without any real-world interaction, as one can also see in Fig.~\ref{fig:all_metrics_all_policies}. The same figure shows the sum of advantages metric that works well in simulation performs poorly in the real-world setting we care about.  Appendix~\ref{app:different_dataset} includes additional experiments showing correlation mostly unchanged on different validation datasets.

Furthermore, we demonstrate that SoftOPC can track the performance of a policy acting in the simulated grasping environment, as it is training in Fig.~\ref{fig:correlation_scatter}, which could potentially be useful for early stopping. 
Finally, SoftOPC seems to be performing slightly better than OPC in most of the experiments. We believe this occurs because the Q-functions compared in each experiment tend to have similar magnitudes. Preliminary results in Appendix~\ref{app:opc_comparison} suggest that when Q-functions have different magnitudes, \METRICNAME{} might
outperform \QDIFFNAME{}.

%% file: F-conclusion.tex
\vspace{-0.05in}
\section{Conclusion and future work}
We proposed \METRICNAME{} and \QDIFFNAME{}, classification-based off-policy evaluation metrics that can be used together with Q-learning algorithms. Our metrics can be used with binary reward tasks: tasks where each episode results in either a failure (zero return) or success (a return of one). While this class of tasks is a substantial restriction, many practical tasks actually fall into this category, including the real-world robotics tasks in our experiments. The analysis of these metrics shows that it can approximate the expected return in deterministic binary reward MDPs.
Empirically, we find that \METRICNAME{} and the \QDIFFNAME{} variant correlate well with performance across several environments, and predict generalization performance across several scenarios.
including the simulation-to-reality scenario, a critical setting for robotics.
Effective off-policy evaluation is critical for real-world reinforcement learning, where it provides an alternative to expensive real-world evaluations during algorithm development.
Promising directions for future work include developing a variant of our method that is not restricted to binary reward tasks. We include some initial work in Appendix~\ref{app:nonbinary}.
However, even in the binary setting, we believe that methods such as ours can provide for a substantially more practical pipeline for evaluating transfer learning and off-policy reinforcement learning algorithms.

%% file: appendix_neurips.tex
\section*{Appendix for Off-Policy Evaluation via Off-Policy Classification}

\setcounter{section}{0}
\renewcommand\thesection{\Alph{section}}

\section{Classification error bound}

\subsection{Trajectory $\tau$ return $1 \Rightarrow$ all $\ba_t$ \good{}}
\label{app:all-good}

Suppose this were not true. Then, there exists a time $t$ where $\ba_t$ is \bad{}. After executing said $\ba_t$, it should be impossible for $\tau$ to end in a success state. However, we know $\tau$ ends in success. By contradiction, all $(\bs_t, \ba_t)$ should be \good{}.

\subsection{Proof of Theorem 1}
\label{app:error-bound-proof}
To bound $R(\pi)$, it is easier to bound the failure rate $1 - R(\pi)$. Policy $\pi$ succeeds if and only if at every $\bs_t$, it selects a \good{} $\ba_t$ a \good{} $\ba_t$, and dynamics do not take it to a \bad{} $\bs_{t+1}$. The failure rate is the total probability $\pi$ makes its first mistake or is first brought to a \bad{} $\bs_{t+1}$, summed over all $t$.

The distribution $\rho^+_{t,\pi}$ is defined such that $\rho^+_{t,\pi}$ is the marginal state distribution at time $t$, conditioned on $\ba_1, \cdots, \ba_{t-1}$ being \good{} and $\bs_t$ being \good{}. This makes $\epsilon_t$ the expected probability $\pi$ choosing a \bad{} $\ba_t$ at time $t$, given no mistakes earlier in the trajectory. Conditioned on this, the failure rate at time $t$ is upper-bounded by $\epsilon_t + c$.

\vspace{-0.10in}
\begin{align}
    1 - R(\pi) &\le \sum_{t=1}^T p(\pi \text{ at \good{} $\bs_t$})\cdot(\epsilon_t + c) \\
    &\le\sum_{t=1}^T (\epsilon_t + c) \\
    &\le T(\epsilon + c)
\end{align}
This gives $R(\pi) \ge 1 - T(\epsilon+c)$ as desired. This bound is tight when $\pi$ is always at a \good{} $\bs_t$, which occurs when $c = 0$, $\epsilon_1 = \epsilon_2 = \cdots = \epsilon_{T-1} = 0$, and $\epsilon_T = T\epsilon$. When $c > 0$, this bound may be improvable.

\subsection{Alternate proof connecting to behavioral cloning in deterministic case}
\label{app:behavior-clone}
In a deterministic environment, we have $c(\bs, \ba) = 0$ for all $(\bs, \ba)$, and a policy that only picks \good{} actions will always achieve the optimal return of 1. Any such policy can be considered an expert policy $\pi^*$. Since $\epsilon$ is defined as the 0-1 loss over states conditioned on not selecting a \bad{} action, we can view $\epsilon$ as the 0-1 behavior cloning loss to an expert policy $\pi^*$. In this section, we present an alternate proof based on behavioral cloning bounds from~\citet{ross2010efficient}.

Theorem 2.1 of \citet{ross2010efficient} proves a $O(T^2\epsilon)$ cost bound for general MDPs. This differs from the $O(T\epsilon)$ cost derived above. The difference in bound comes because~\citet{ross2010efficient} derive their proof in a general MDP, whose cost is upper bounded by $1$ at every timestep. If $\pi$ deviates from the expert, it receives cost $1$ several times, once for every future timestep. In binary reward MDPs, we only receive this cost once, at the final timestep. Transforming the proof to incorporate our binary reward MDP assumptions lets us recover the $O(T\epsilon)$ upper bound from Appendix~\ref{app:error-bound-proof}. We briefly explain the updated proof, using notation from \citep{ross2010efficient} to make the connection more explicit.

Define $\epsilon_t$ as the expected 0-1 loss at time $t$ for $\pi$ under the state distribution of $\pi^*$. Since $\rho^+_{t,\pi}$ corresponds to states $\pi$ visits conditioned on never picking a \bad{} action, this is the same as our definition of $\epsilon_t$. The MDP is defined by cost instead of reward: cost $C(\bs)$ of state $\bs$ is $0$ for all timesteps except the final one, where $C(\bs) \in \{0,1\}$. Let $p_{t}$ be the probability $\pi$ hasn't made a mistake (w.r.t $\pi^*$) in the first $t$ steps, $d_t$ be the state distribution conditioned on no mistakes in the first $t-1$ steps, and $d'_t$ be the state distribution conditioned on $\pi$ making at least 1 mistake. In a general MDP with $0 \le C(\bs) \le 1$, total cost $J(\pi)$ is bounded by \mbox{$J(\pi) \le \sum_{t=1}^T [p_{t-1}\EX{d_t}{C_\pi(\bs)} + (1-p_{t-1})]$}, where the 1st term is cost while following the expert and the 2nd term is an upper bound of cost $1$ if outside of the expert distribution. In a binary reward MDP, since $C(\bs) = 0$ for all $t$ except $t = T$, we can ignore every term in the summation except the final one, giving

\begin{equation}
    J(\pi) = p_{T-1}\EX{d_T}{C_\pi(\bs_T)} + (1-p_{T-1})
\end{equation}

Note $\EX{d_T}{C_\pi(\bs_T)} = \epsilon_T$, and as shown in the original proof, $p_t \ge 1 - \sum_{i=1}^t \epsilon_i$. Since $p_{T-1}$ is a probability, we have $p_{T-1}\EX{d_T}{C_\pi(\bs_T)} \le \EX{d_T}{C_\pi(\bs_T)}$, which recovers the $O(T\epsilon)$ bound, and again this is tight when $\epsilon_1 = \cdots = \epsilon_{T-1} = 0, \epsilon_T = T\epsilon$.

\begin{equation}
    J(\pi) \le \EX{d_T}{C_\pi(\bs_T)} + \sum_{t=1}^{T-1} \epsilon_t = \sum_{t=1}^T \epsilon_T = T\epsilon
\end{equation}

\section{Algorithm pseudocode}
\label{app:efficient-compute}

We first present pseudocode for \QDIFFNAME{}.
\begin{equation}
    \text{\QDIFFNAME{}}(Q) = p(y=1) \EX{(\bs,\ba),y=1}{Q(\bs,\ba)} - \EX{(\bs,\ba)}{Q(\bs,\ba)}.
\end{equation}

\begin{algorithm}
	\caption{Soft Off-Policy Classification (SoftOPC)} 
	\begin{algorithmic}[1]
	    \Require Dataset $\mathcal{D}$ of trajectories $\tau=(\bs_1, \ba_1, \ldots, \bs_T, \ba_T)$, learned Q-function $Q(\bs, \ba)$, prior $p(y=1)$ (set to $1$ for all experiments).
	    \State PositiveAverages $\gets EMPTY\_LIST$
	    \State AllAverages $\gets EMPTY\_LIST$
		\For {$(\bs_1, \ba_1, \ldots, \bs_T, \ba_T, r_T) \in \mathcal{D}$}
		    \State Compute Q-values $Q(\bs_1, \ba_1), \ldots, Q(\bs_T, \ba_T)$.
		    \State AverageQ $\gets \frac{1}{T}\sum_t Q(\bs_t, \ba_t)$
		    \State AllAverages.append(AverageQ)
		    \If {$r_T = 1$}
		        \State PositiveAverages.append(AverageQ)
		    \EndIf
		\EndFor
		\State \Return $p(y=1)AVERAGE(\text{PositiveAverages}) - AVERAGE(\text{AllAverages})$
	\end{algorithmic}
	\label{alg:softopc}
\end{algorithm}

Next, we present pseudocode for \METRICNAME{}.
\begin{equation}
    \text{\METRICNAME{}}(Q) = p(y=1) \EX{(\bs,\ba),y=1}{1_{Q(s,a)>b}} - \EX{(\bs,\ba)}{1_{Q(s,a)>b}}
\end{equation}

Suppose we have $N$ transitions, of which $N^+$ of them have positive labels. Imagine placing each $Q(\bs,\ba)$ on the number line. Each $Q(\bs,\ba)$ is annotated with a score, $-1/N$ for unlabeled transitions and $p(y=1)/N^+ - 1/N$ for positive labeled transitions. Imagine sliding a line from $b = -\infty$ to $b = \infty$. At $b = -\infty$, the \METRICNAME{} score is $p(y=1) - 1$. The \METRICNAME{} score only updates when this line passes some $Q(\bs,\ba)$ in our dataset, and is updated based on the score annotated at $Q(\bs,\ba)$. After sorting by Q-value, we can find all possible \METRICNAME{} scores in $O(N)$ time by moving $b$ from $-\infty$ to $\infty$, noting the updated score after each $Q(\bs,\ba)$ we pass. Given $\mathcal{D}$, we sort the $N$ transitions in $O(N \log N)$, annotate them appropriately, then compute the maximum over all \METRICNAME{} scores. When doing so, we must be careful to detect when the same $Q(\bs,\ba)$ value appears multiple times, which can occur when $(\bs, \ba)$ appears in multiple trajectories.

\begin{algorithm}
	\caption{Off-Policy Classification (\METRICNAME{})} 
	\begin{algorithmic}[1]
	    \Require Dataset $\mathcal{D}$ of trajectories $\tau=(\bs_1, \ba_1, \ldots, \bs_T, \ba_T)$, learned Q-function $Q(\bs, \ba)$, prior $p(y=1)$ (set to $1$ for all experiments).
	    \State $N^+ \gets$ number of $(\bs,\ba) \in \mathcal{D}$ from trajectories where $r_T = 1$
	    \State $N \gets$ number of $(\bs, \ba) \in \mathcal{D}$
	    \State Q-values $\gets EMPTY\_DICTIONARY$
		\For {$(\bs_1, \ba_1, \ldots, \bs_T, \ba_T, r_T) \in \mathcal{D}$} \Comment{Prepare annotated number line}
		    \State Compute Q-values $Q(\bs_1, \ba_1), \ldots, Q(\bs_T, \ba_T)$.
		    \For {$t = 1, 2, \ldots, T$}
		        \If {$Q(\bs,\ba) \not\in\,$ Q-values.keys()}
		            \State Q-values[$Q(\bs,\ba)$] = 0
		        \EndIf
		        \If {$r_T = 1$}
		            \State Q-values[$Q(\bs,\ba)$] $\mathrel{+}= p(y=1)/N^+$
		        \Else
		            \State Q-values[$Q(\bs,\ba)$] $\mathrel{+}= -1/N$
		        \EndIf
		    \EndFor
		\EndFor
		\State RunningTotal $\gets p(y=1) - 1$ \Comment{\METRICNAME{} when $b = -\infty$}
		\State Best\METRICNAME{} $\gets$ RunningTotal
		\For {$b \in \,$Sorted(Q-values.keys())}
		    \State RunningTotal += Q-values[$b$]
		    \State Best\METRICNAME{} $\gets \max$(Best\METRICNAME{}, RunningTotal)
		\EndFor
		\State \Return Best\METRICNAME{}
	\end{algorithmic}
	\label{alg:opc}
\end{algorithm}

\section{Baseline metrics}
\label{app:baseline_impl}

We elaborate on the exact expressions used for the baseline metrics. In all baselines, ${\bf a}_{t}^\pi$ is the on-policy action $\argmax_{\bf a}Q^\pi({\bf s}_t, {\bf a})$.

\paragraph{Temporal-difference error}
The TD error is the squared error between $Q(\bs,\ba)$ and the 1-step return estimate of the action's value.
\begin{equation}
\label{eq:basic_error}
     \EX{\bs_t,\ba_t \sim \pi_b}{\left( Q^\pi({\bf s}_t, {\bf a}_t) -(r_t + \gamma Q^\pi(\bs_{t+1},{\bf a}_{t}^\pi))\right)^2}
\end{equation} 

\paragraph{Discounted sum of advantages}
The difference of the value functions of two policies $\pi$ and $\pi_b$ at state $\bs_t$ is given by the discounted sum of advantages~\citep{kakade2002, Murphy2005-qc} of $\pi$ on episodes induced by $\pi_b$: 
\begin{equation}
    V^{\pi_b}(\bs_t) - V^{\pi}(\bs_t) = \E_{\bs_t,\ba_t\sim\pi_b} \left [ \sum_{t'=t}^T \gamma^{t'-t}A^\pi(\bs_{t'}, \ba_{t'}) \right ],
    \label{eq:murphy}
\end{equation}
where $\gamma$ is the discount factor and $A^\pi$ the advantage function for policy $\pi$, defined as $A^\pi({\bf s}_t, {\bf {\bf a}_t}) = Q^\pi({\bf s}_t,{\bf a}_t) - Q^\pi({\bf s}_t,\ba_t^{\pi})$. Since $V^{\pi_b}$ is fixed, estimating (\ref{eq:murphy}) is sufficient to compare $\pi_1$ and $\pi_2$. The $\pi$ with smaller score is better.
\begin{equation}
    \E_{\bs_t,\ba_t \sim \pi_b} \left [ \sum_{t'=t}^T \gamma^{t'-t}A^\pi(\bs_{t'}, \ba_{t'}) \right ].
    \label{eq:murphy_metric}
\end{equation}

\paragraph{Monte-Carlo estimate corrected with the discounted sum of advantages}

Estimating $V^{\pi_b}(\bs_t) = \E_{\pi_b}[\sum_{t'}\gamma^{t'-t}r_{t'}]$ with the Monte-Carlo return, substituting into Eqn. (\ref{eq:murphy}), and rearranging gives
\begin{equation}
    V^{\pi}(\bs_t) = \E_{\bs_t,\ba_t\sim\pi_b} \left [ \sum_{t'=t}^T \gamma^{t'-t}\left(r_{t'} - A^\pi(\bs_{t'}, \ba_{t'})\right) \right ]
\end{equation}

With $V^\pi(\bs_t) + A^\pi(\bs_t,\ba_t) = Q^\pi(\bs_t, \ba_t)$, we can obtain an approximate ${\widetilde Q}$ %
estimate depending on the whole episode:
\begin{equation}
    {\widetilde Q}_{MCC}({\bf s}_t,{\bf a}_t,\pi) = \E_{\bs_t,\ba_t\sim\pi_b} \left [ r_t +   \sum_{t'=t+1}^T \gamma^{t'-t}(r_{t'}-A^\pi(\bs_{t'}, {\bf a}_{t'})) \right ]
   \label{eq:mcc1}
\end{equation}
The MCC Error is the squared error to this estimate.
\begin{equation}
\label{eq:mcc}
     \EX{\bs_t,\ba_t \sim \pi_b}{\left( Q^\pi({\bf s}_t, {\bf a}_t) -{\widetilde Q}_{MCC}({\bf s}_t,{\bf a}_t,\pi)\right)^2}
\end{equation} 

Note that (\ref{eq:mcc}) was proposed before by \citet{Quillen2018-ms} as a training loss for a Q-learning variant, but not for the purpose of off-policy evaluation.

Eqn. (\ref{eq:basic_error}) and Eqn. (\ref{eq:mcc1}) share the same optimal Q-function, so assuming a perfect learning algorithm, there is no difference in information between these metrics. In practice, the Q-function will not be perfect due to errors in function approximation and optimization. Eqn. (\ref{eq:mcc1}) is designed to rely on all future rewards from time $t$, rather than just $r_t$. We theorized that using more of the ``ground truth'' from $\mathcal{D}$ could improve the metric's performance in imperfect learning scenarios. This did not occur in our experiments - the MCC Error performed poorly.

\section{Argument for choosing $p(y=1) = 1$}
\label{app:prior}
The positive class prior $p(y=1)$ should intuitively depend on the environment, since some environments will have many more \good{} $(\bs,\ba)$ than others. However, recall how error $\epsilon$ is defined.
Each $\epsilon_t$ is defined as:
\begin{equation}
    \epsilon_t = \EX{\rho^+_{t,\pi}}{\sum_{\ba \in \mathcal{A}_-(\bs_t)} \pi(\ba|\bs_t)}
\end{equation}
where state distribution $\rho^+_{t,\pi}$ is defined such that $\ba_1,\cdots,\ba_{t-1}$ were all \good{}. This is equivalent to following an optimal ``expert'' policy $\pi^*$, and although we are estimating $\epsilon_t$ from data generated by behavior policy $\pi_b$, we should match the positive class prior $p(y=1)$ we would observe from expert $\pi^*$. Assuming the task is feasible, meaining the policy has \good{} actions available from the start, we have $R(\pi^*) = 1$. Therefore, although the validation dataset will likely have both successes and failures, a prior of $p(y=1) = 1$ is the ideal prior, and this holds independently of the environment.
As a didactic toy example, we show this holds for a binary tree domain. In this domain, each node is a state, actions are $\{left,right\}$, and leaf nodes are terminal with reward $0$ or $1$. We try \mbox{$p(y=1) \in \{0, 0.05, 0.1, \cdots, 0.9, 0.95, 1\}$} in two extremes: only 1 leaf fails, or only 1 leaf succeeds. No stochasticity was added. Validation data is from the uniformly random policy. The frequency of \good{} $(\bs,\ba)$ varies a lot between the two extremes, but in both Spearman correlation $\rho$ monotonically increases with $p(y=1)$ and was best with $p(y=1) = 1$.
Fig.~\ref{fig:tree_prior} shows Spearman correlation of \METRICNAME{} and \QDIFFNAME{} with respect to $p(y=1)$, when the tree is mostly success or failures. In both settings $p(y=1) = 1$ has the best correlation.

\begin{figure}[h]
    \centering
    \includegraphics[width=0.49\linewidth]{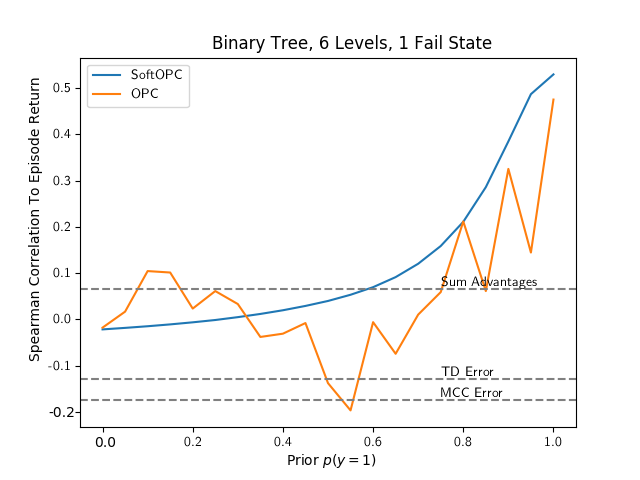}
    \includegraphics[width=0.49\linewidth]{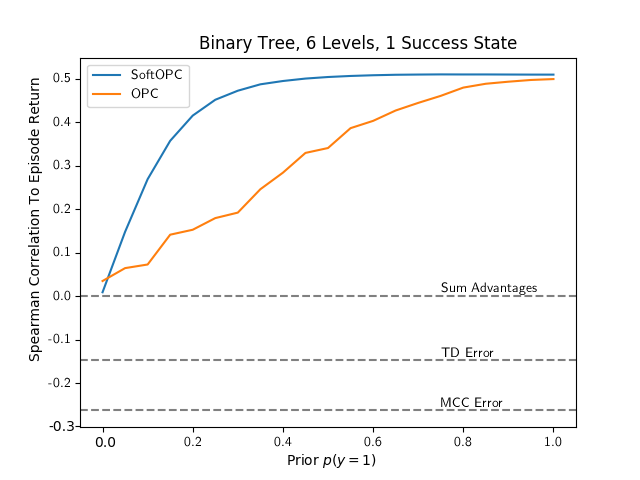}
    \caption{Spearman correlation of \QDIFFNAME{}, \METRICNAME{}, and baselines with varying $p(y=1)$. Baselines do not depend on $p(y=1)$. Correlations further from 0 are better.}
    \label{fig:tree_prior}
\end{figure}

From an implementation perspective, $p(y=1) = 1$ is also the only choice that can be applied across arbitrary validation datasets. Suppose $\pi_b$, the policy collecting our validation set, succeeds with probability $R(\pi_b) = p$.
In the practical computation of \METRICNAME{} presented in Algorithm~\ref{alg:opc}, we have $N$ transitions, and $pN$ of them have positive labels. Each $Q(\bs,\ba)$ is annotated with a score: $\frac{-1}{N}$ for unlabeled transitions and $\frac{p(y=1)}{pN} - \frac{1}{N}$ for positive labeled transitions. The maximal \METRICNAME{} score will be the sum of all annotations within the interval $[b,\infty)$, and we maximize over $b$.

For unlabeled transitions, the annotation is $-1/N$, which is negative. Suppose the annotation for positive transitions was negative as well. This occurs when $p(y=1)/(pN) - 1/N < 0$. If every annotation is negative, then the optimal choice for $b$ is $b = \infty$, since the empty set has total 0 and every non-empty subset has a negative total. This gives $\text{\METRICNAME{}}(Q) = 0$, no matter what $Q(\bs,\ba)$ we are evaluating, which makes the \METRICNAME{} score entirely independent of episode return.

This degenerate case is undesirable, and happens when $p(y=1)/(pN) < 1/N$, or equivalently $p(y=1) < p$.
To avoid this, we should have $p(y=1) \ge p = R(\pi_b)$. If we wish to pick a single $p(y=1)$ that can be applied to data from arbitrary behavior policies $\pi_b$, then we should pick $p(y=1) \ge R(\pi^*)$. In binary reward MDPs where $\pi^*$ can always succeed, this gives ${p(y=1) \ge R(\pi^*) = 1}$, and since the prior is a probability, it should satisfy $0 \le p(y=1) \le 1$, leaving $p(y=1) = 1$ as the only option.

To complete the argument, we must handle the case where we have a binary reward MDP where $R(\pi^*) < 1$. In a binary reward MDP, the only way to have $R(\pi^*) < 1$ is if the sampled initial state $\bs_1$ is one where $(\bs_1, \ba)$ is \bad{} for all $\ba$. From these $\bs_1$, and all future $\bs_t$ reachable from $\bs_1$, the actions $\pi$ chooses do not matter - the final return will always be $0$. Since $\epsilon_t$ is defined conditioned on only executing \good{} actions so far, it is reasonable to assume we only wish to compute the expectation over states where our actions can impact reward.
If we computed optimal policy return $R(\pi^*)$ over just the initial states $\bs_1$ where \good{} actions exist, we have $R(\pi^*) = 1$, giving $1 \le p(y=1) \le 1$ once again.

\section{Experiment details}

\subsection{Binary tree environment details}
\label{app:binary_details}

The binary tree is a full binary tree with $k = 6$ levels. The initial state distribution is uniform over all non-leaf nodes. Initial state may sometimes be initialized to one where failure is inevitable. The validation dataset is collected by generating 1,000 episodes from the uniformly random policy. For Q-functions, we generate 1,000 random Q-functions by sampling $Q(\bs, \ba) \sim U[0,1]$ for every $(\bs, \ba)$, defining the policy as $\pi(\bs) = \argmax_\ba Q(\bs,\ba)$. We try priors $p(y=1) \in \{0, 0.05, 0.1, \cdots, 0.9, 0.95, 1\}$. %
Code for this environment is available at \url{https://gist.github.com/alexirpan/54ac855db7e0d017656645ef1475ac08}.

\subsection{Pong details}
\label{app:pong_details}

Fig.~\ref{fig:atari_results_labeled} is a scatterplot of our Pong results. Each color represents a different hyperparameter setting, as explained in the legend. From top to bottom, the abbreviations in the legend mean:

\begin{itemize}
    \item \code{DQN}: trained with DQN
    \item \code{DDQN}: trained with Double DQN
    \item \code{DQN\_gamma9}: trained with DQN, $\gamma=0.9$ (default is $0.99$).
    \item \code{DQN2}: trained with DQN, using a different random seed
    \item \code{DDQN2}: trained with Double DQN, using a different random seed
    \item \code{DQN\_lr1e4}: trained with DQN, learning rate $10^{-4}$ (default is $2.5 \times 10^{-4}$).
    \item \code{DQN\_b64}: trained with DQN, batch size 64 (default is 32).
    \item \code{DQN\_fixranddata}: The replay buffer is filled entirely by a random policy, then a DQN model is trained against that buffer, without pushing any new experience into the buffer.
    \item \code{DDQN\_fixranddata}: The replay buffer is filled entirely by a random policy, then a Double DQN model is trained against that buffer, without pushing new experience into the buffer.
\end{itemize}

In Fig.~\ref{fig:atari_results_labeled}, models trained with $\gamma=0.9$ are highlighted. We noticed that \QDIFFNAME{} was worse at separating these models than \METRICNAME{}, suggesting the 0-1 loss is preferable in some cases. This is discussed further in Appendix~\ref{app:opc_comparison}.

In our implementation, all models were trained in the full version of the Pong game, where the maximum return possible is $21$ points. However, to test our approach we create a binary version for evaluation. Episodes in the validation set were truncated after the first point was scored. Return of the policy was estimated similarly: the policy was executed until the first point is scored, and the average return is computed over these episodes. Although the train time environment is slightly different from the evaluation environment, this procedure is fine for our method, since our method can handle environment shift and we treat $Q(\bs,\ba)$ as a black-box scoring function. \METRICNAME{} can be applied as long as the validation dataset matches the semantics of the test environment where we estimate  the final return.

\begin{figure}[h]
    \centering
    \includegraphics[width=0.48\linewidth]{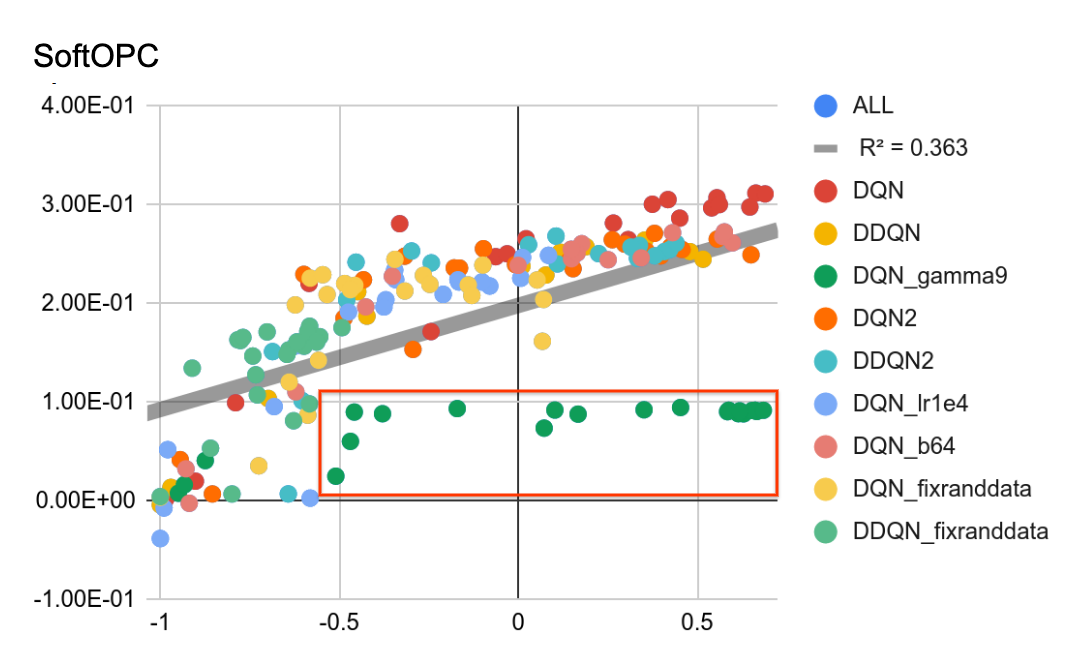}
    \includegraphics[width=0.48\linewidth]{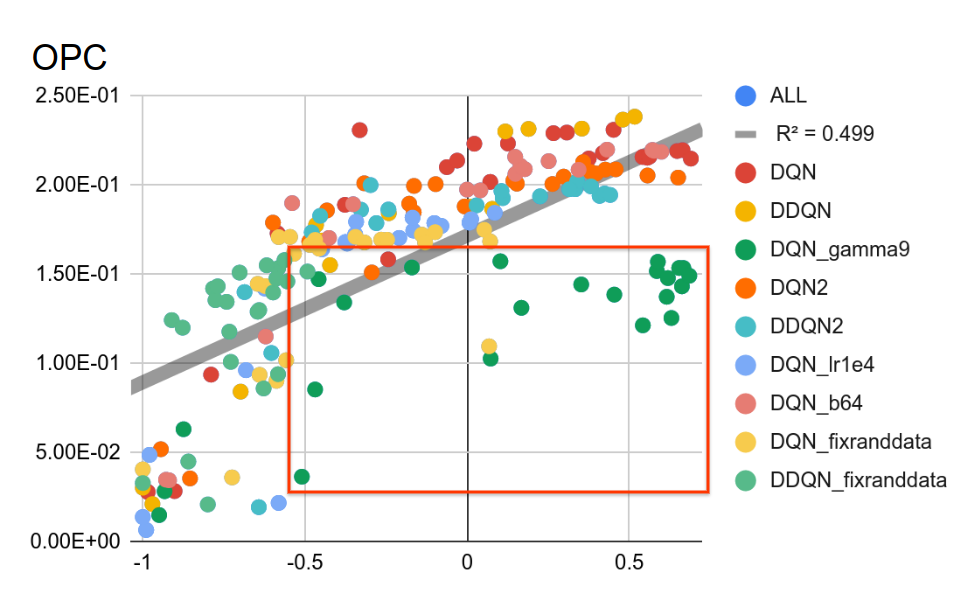}
    \caption{Scatterplot of episode return (x-axis) of Pong models against metric (y-axis), for \QDIFFNAME{} (left) and \METRICNAME{} (right). Each color is a set of model checkpoints from a different hyperparameter setting, with the legend explaining the mapping from color to hyperparameters. In each plot, points trained with DQN, $\gamma=0.9$ are boxed with a red rectangle. We observed that the hard 0-1 loss in \METRICNAME{} does a better job separating these models than the soft loss in \QDIFFNAME{}.}
    \label{fig:atari_results_labeled}
\end{figure}

\subsection{Simulated grasping details}
\label{app:sim_details}

The objects we use were generated randomly through procedural generation. The resulting objects are highly irregular and are often non-convex. Some example objects are shown in Fig.~\ref{fig:object_examples}.

\begin{figure}[h]
    \centering
    \begin{subfigure}[b]{0.45\linewidth}
    \includegraphics[width=\linewidth]{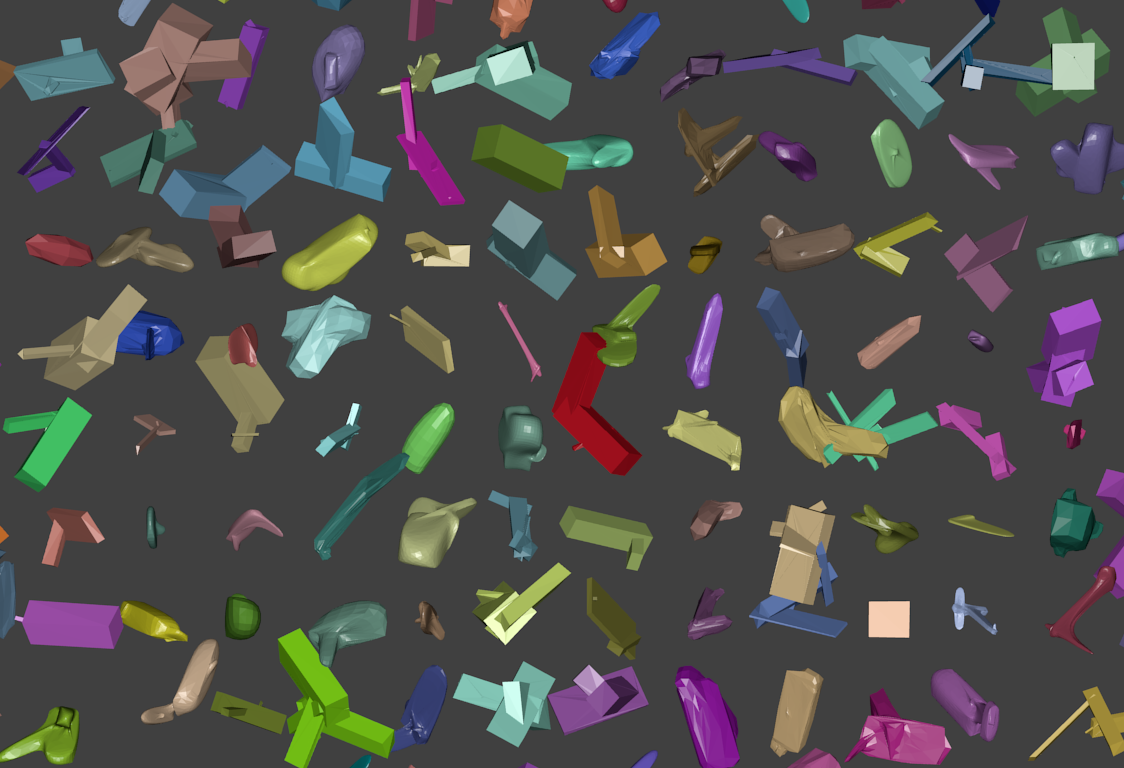}
    \caption{Procedurally generated objects for simulation}
     \label{fig:object_examples}
    \end{subfigure}
    \begin{subfigure}[b]{0.45\linewidth}
    \includegraphics[width=\linewidth]{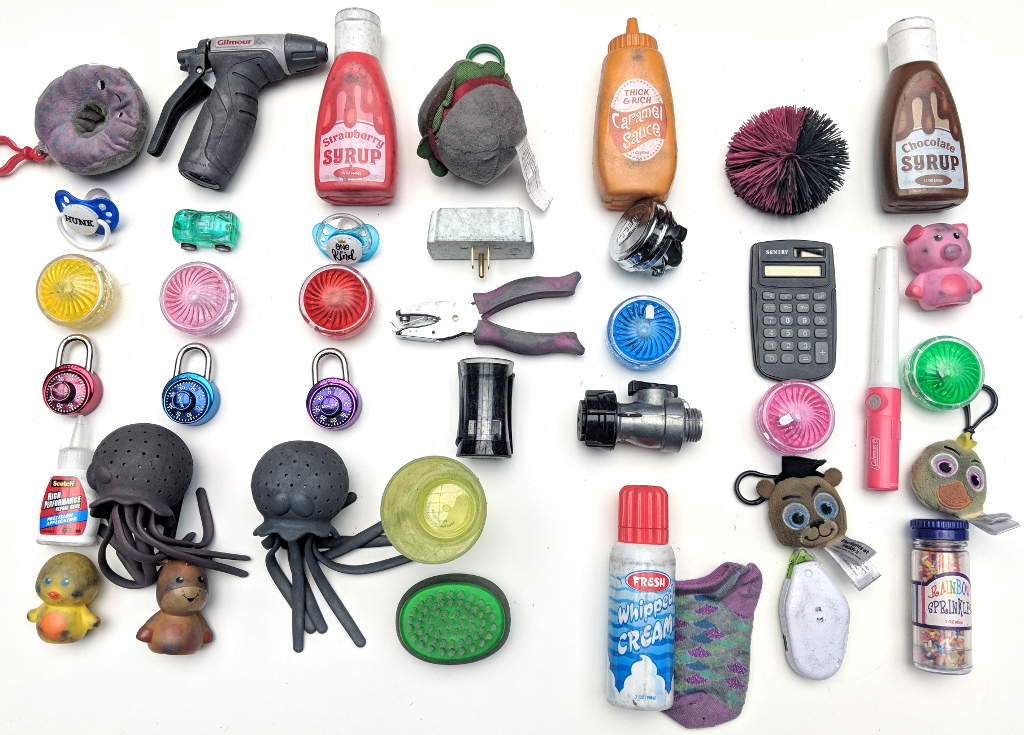}
    \caption{Test objects for real-world evaluation}
    \label{fig:real_object_examples}
    \end{subfigure}
    \caption{\textbf{(a):} Example objects from the procedural generation process used during training in simulation. \textbf{(b):} Real test objects
    used during evaluation on real robots.}
\end{figure}

Fig.~\ref{fig:bellman_overfit} demonstrates two generalization problems from Sect.~\ref{sec:generalization}: \textit{insufficient off-policy training data} and \textit{mismatched off-policy training data}.
We trained two models with a limited 100k grasps dataset or a large 900k grasps dataset, then evaluated grasp success.
The model with limited data fails to achieve stable grasp success due to overfitting to its limited dataset. Meanwhile, the model with abundant data
learns to model the train objects, but fails to model the test objects, since they are unobserved at training time. The train and test objects used are show in Fig.~\ref{fig:train_test_objects}.

\begin{figure}[h]
    \centering
    \includegraphics[width=0.329\linewidth]{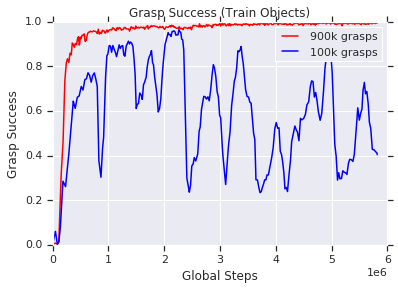}
    \includegraphics[width=0.329\linewidth]{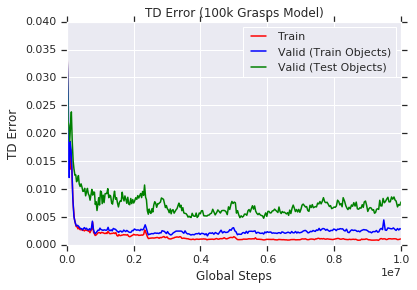}
    \includegraphics[width=0.329\linewidth]{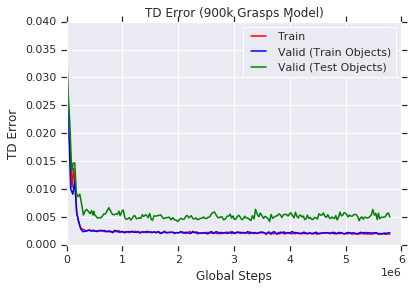}
    \caption{%
    \textit{Left:} Grasp success curve of model trained with 900k or 100k grasps. The 100k grasps model oscillates in performance. \textit{Middle:} We see why: holdout TD error (blue) of the 100k grasps model is increasing. \textit{Right}: The TD Error for the 900k grasp model is the same for train and holdout, but is still larger for test data on unseen test objects.}
    \label{fig:bellman_overfit}
\end{figure}

\begin{figure}[h]
\begin{subfigure}[b]{0.45\linewidth}
\includegraphics[width=0.19\linewidth]{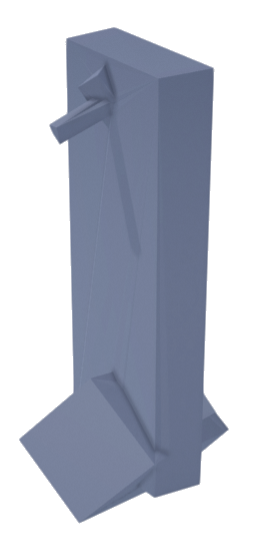}
\includegraphics[width=0.19\linewidth]{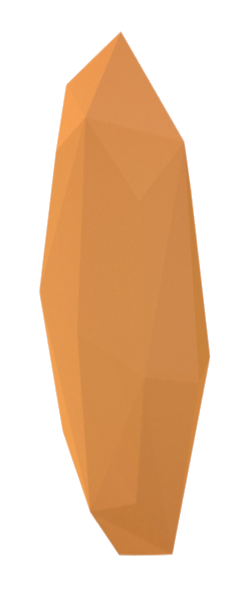}
\includegraphics[width=0.19\linewidth]{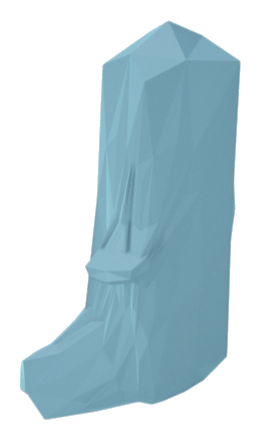}
\includegraphics[width=0.19\linewidth]{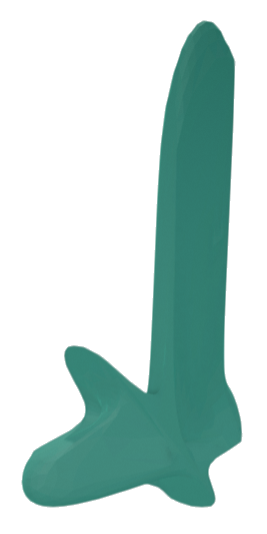}
\includegraphics[width=0.19\linewidth]{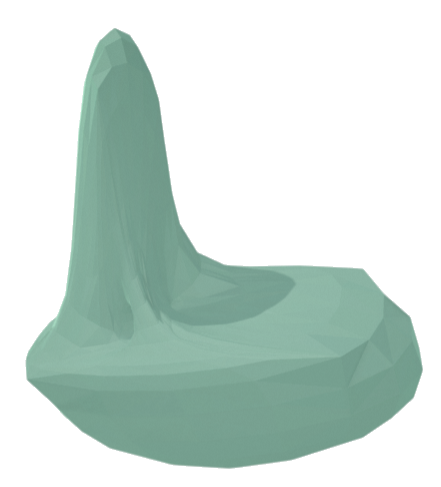}
\caption{Train objects}
\end{subfigure}
\hfill
\begin{subfigure}[b]{0.45\linewidth}
\includegraphics[width=0.19\linewidth]{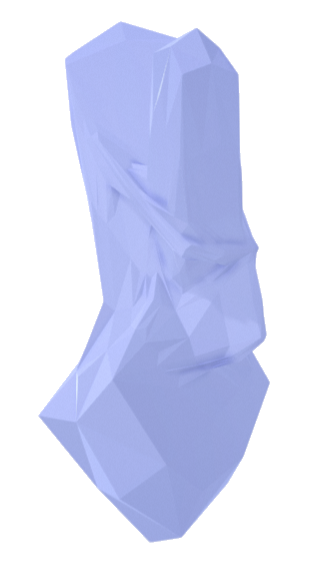}
\includegraphics[width=0.19\linewidth]{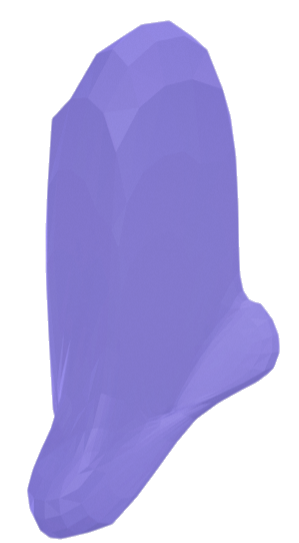}
\includegraphics[width=0.19\linewidth]{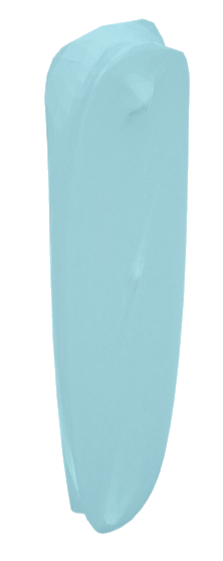}
\includegraphics[width=0.19\linewidth]{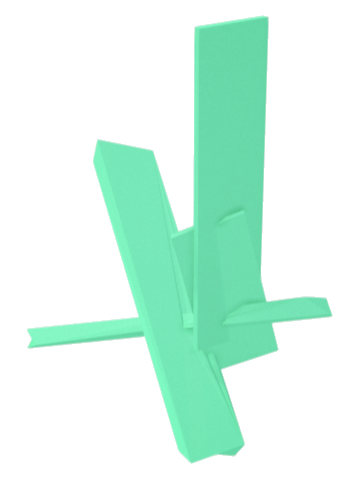}
\includegraphics[width=0.19\linewidth]{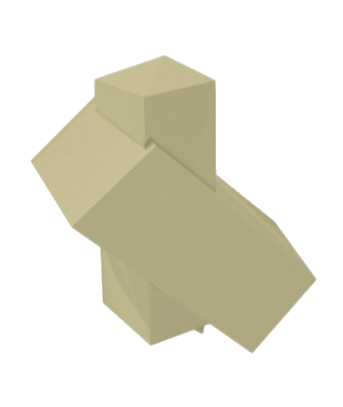}
\caption{Test objects}
\end{subfigure}
\caption{\textbf{Example of mismatched off-policy training data.} Train objects and test objects from simulated grasping task in Sect.~\ref{sec:sim-robot}. Given large amounts of data from train objects, models do not fully generalize to test objects.}
\label{fig:train_test_objects}
\end{figure}

\subsection{Real-world grasping}
\label{app:real_details}
\begin{figure}[h]
    \centering
    \includegraphics[width=\linewidth]{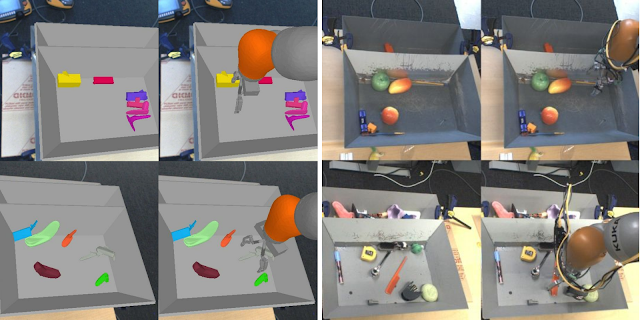}
    \caption{Visual differences for the robot grasping task between simulation (left) and reality (right). In simulation, models are trained to grasp procedurally generated shapes while in reality objects have more complex shapes. The simulated robot arm does not have the same colors and textures as the real robot and lacks the visible real cables. The tray in reality has a greater variation in appearance than in the simulation.}
    \label{fig:kuka_real_v_sim}
\end{figure}

Several visual differences between simulation and reality limit the real performance of model trained in simulation (see Fig.~\ref{fig:kuka_real_v_sim}) and motivate simulation-to-reality methods such as the Randomized-to-Canonical Adaptation Networks (RCANs), as proposed by~\citet{james2019rcan}. The 15 real-world grasping models evaluated were trained using variants of RCAN. These networks train a generator to transform randomized simulation images to a canonical simulated image. A policy is learned over this canonical simulated image. At test time, the generator transforms real images to the same canonical simulated image, facilitating zero-shot transfer. Optionally, the policy can be fine-tuned with real-world data, in this case the real-world training objects are distinct from the evaluation objects. The \QDIFFNAME{} and real-world grasp success of each model is listed in Table~\ref{tab:real-world-details}. From top-to-bottom, the abbreviations mean:
\begin{itemize}
    \item        \code{Sim:} A model trained only in simulation.
    \item        \code{Randomized Sim:} A model trained only in simulation with the \textit{mild randomization} scheme from~\citet{james2019rcan}: random tray texture, object texture and color, robot arm color, lighting direction and brightness, and one of 6 background images consisting
    of 6 different images from the view of the real-world camera. 
    \item        \code{Heavy Randomized Sim:} A model trained only in simulation with the \textit{heavy randomization} scheme from~\citet{james2019rcan}: every randomization from \textit{Randomized Sim}, as well as slight randomization of the position of the robot arm and tray, randomized position of the divider within the tray (see Figure 1b in main text for a visual of the divider), and a more diverse set of background images.
    \item        \code{Randomized Sim + Real (2k):} The \code{Randomized Sim} Model, after training on an additional 2k grasps collected on-policy on the real robot.
    \item        \code{Randomized Sim + Real (3k):} The \code{Randomized Sim} Model, after training on an additional 3k grasps collected on-policy on the real robot.
    \item        \code{Randomized Sim + Real (4k):} The \code{Randomized Sim} Model, after training on an additional 4k grasps collected on-policy on the real robot.
    \item        \code{Randomized Sim + Real (5k):} The \code{Randomized Sim} Model, after training on an additional 5k grasps collected on-policy on the real robot.
    \item        \code{RCAN:} The RCAN model, as described in~\cite{james2019rcan}, trained in simulation with a pixel level adaptation model.
    \item        \code{RCAN + Real (2k):} The \code{RCAN} model, after training on an additional 2k grasps collected on-policy on the real robot. 
    \item        \code{RCAN + Real (3k):} The \code{RCAN} model, after training on an additional 3k grasps collected on-policy on the real robot.
    \item        \code{RCAN + Real (4k):} The \code{RCAN} model, after training on an additional 4k grasps collected on-policy on the real robot.
    \item        \code{RCAN + Real (5k):} The \code{RCAN} model, after training on an additional 5k grasps collected on-policy on the real robot.
    \item        \code{RCAN + Dropout:} The \code{RCAN} model with dropout applied in the policy.
    \item        \code{RCAN + InputDropout:} The \code{RCAN} model with dropout applied in the policy and RCAN generator.
    \item        \code{RCAN + GradsToGenerator:} The \code{RCAN} model where the policy and RCAN generator are trained simultaneously, rather than training RCAN first and the policy second.
\end{itemize}

\begin{table}[h]
    \centering
    \begin{tabular}{lcc}
            \textbf{Model}& \textbf{\QDIFFNAME{}} & \textbf{Real Grasp Success (\%)} \\\hline
            Sim		&0.056 &	16.67\\
            Randomized Sim	&	0.072	&	36.92 \\
            Heavy Randomized Sim&		0.040&	34.90\\
            Randomized Sim + Real (2k)		&	0.129	&	72.14 \\
            Randomized Sim + Real (3k)	&	0.141&	73.65 \\
            Randomized Sim + Real (4k)	&		0.149&	82.92\\
            Randomized Sim + Real (5k)	&		0.152&	84.38\\
            RCAN		&		0.113&	65.69\\
            RCAN + Real (2k)		&	0.156&	86.46\\
            RCAN + Real (3k)		&	0.166&	88.34\\
            RCAN + Real (4k)		&	0.152&	87.08\\
            RCAN + Real (5k)		&	0.159&	90.71\\
            RCAN + Dropout	&		0.112&	51.04\\
            RCAN + InputDropout		&	0.089&	57.71\\
            RCAN + GradsToGenerator	&	0.094	&	58.75\\
    \end{tabular}
    
    \caption{Real-world grasping models used for Sect.~\ref{sec:sim-robot} simulation-to-reality experiments.}
    \label{tab:real-world-details}
\end{table}

\section{\QDIFFNAME{} performance on different validation datasets}
\label{app:different_dataset}

For real grasp success we use 7 KUKA LBR IIWA robot arms to each make 102 grasp attempts from 7 different bins with test objects (see Fig.~\ref{fig:real_object_examples}).
Each grasp attempt is allowed up to 20 steps and any grasped object is dropped back in the bin, a successful grasp is made if any of the test objects is held in the gripper at the end of the episode.

For estimating \QDIFFNAME{}, we use a validation dataset collected from two policies, a poor policy with a success of 28\%, and a better policy with a success of 51\%. We divided the validation dataset based on the policy used, then evaluated \QDIFFNAME{} on data from only the poor or good policy. Fig.~\ref{fig:metric_by_policies} shows the correlation on these subsets of the validation dataset. The correlation is slightly worse on the poor dataset, but the relationship between \QDIFFNAME{} and episode reward is still clear.

\begin{figure}[h]
\centering
        \includegraphics[width=0.328\linewidth]{images/metric_softopc_all_policies.png}
        \includegraphics[width=0.328\linewidth]{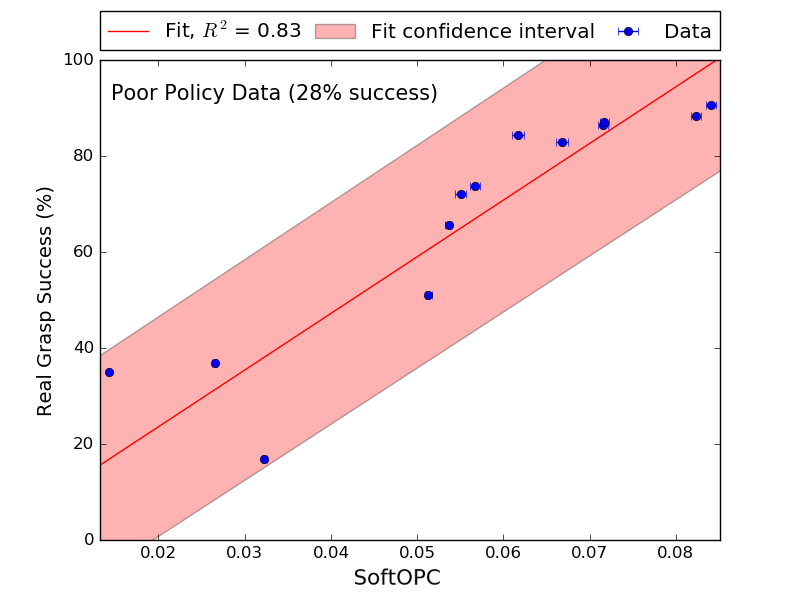}
        \includegraphics[width=0.328\linewidth]{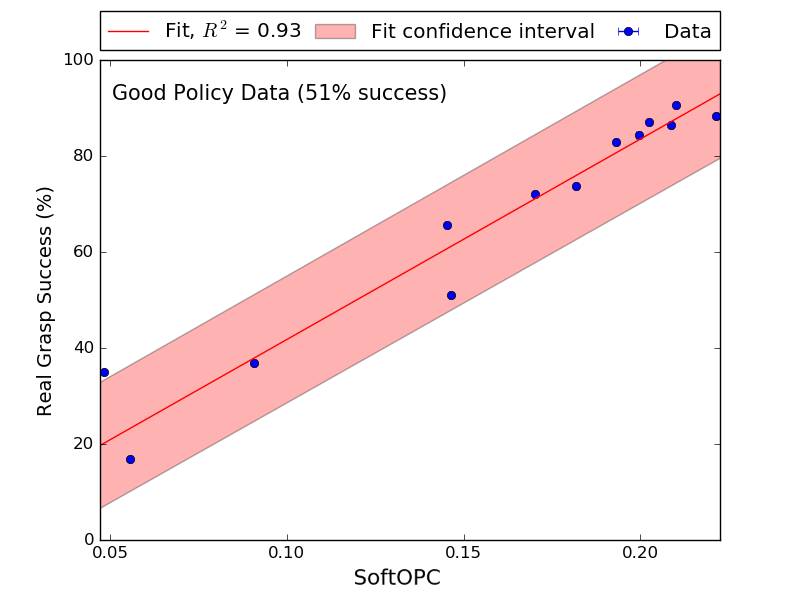}
    \caption{\textbf{\QDIFFNAME{} versus the real grasp success over different validation datasets for Real-World grasping.}
    \textit{Left}: \QDIFFNAME{} over entire validation dataset.
    \textit{Middle}: \QDIFFNAME{} over validation data from only the poor policy (28\% success rate). \textit{Right}: \QDIFFNAME{} over validation data from only the better policy (51\% success). In each, a fitted regression
    line with its $R^2$ and 95\% confidence interval is also shown.}\label{fig:metric_by_policies}
\end{figure}

As an extreme test of robustness, we go back to the simulated grasping environment. We collect a new validation dataset, using the same human-designed policy with $\epsilon=0.9$ greedy exploration instead. The resulting dataset is almost all failures, with only 1\% of grasps succeeding. However, this dataset also covers a broad range of states, due to being very random. Fig.~\ref{fig:sim_metrics_baddata} shows the \METRICNAME{} and \QDIFFNAME{} still perform reasonably well, despite having very few positive labels. From a practical standpoint, this suggests that \METRICNAME{} or \QDIFFNAME{} have some robustness to the choice of generation process for the validation dataset.

\begin{figure}[h]
\centering
        \includegraphics[width=0.48\linewidth]{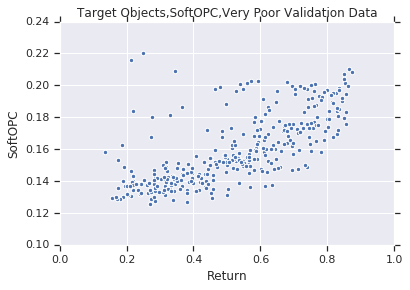}
        \includegraphics[width=0.48\linewidth]{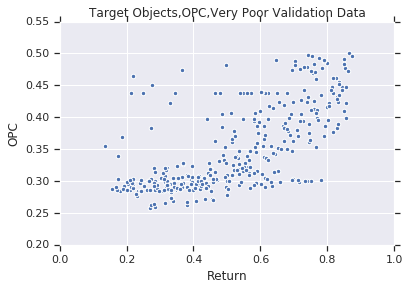}
    \caption{\textbf{\QDIFFNAME{} and \METRICNAME{} over almost random validation data on test objects in simulated grasping.} We generate a validation dataset from an $\epsilon$-greedy policy where $\epsilon = 0.9$, leading to a validation dataset where only 1\% of episodes succeed.
    \textit{Left}: \QDIFFNAME{} over the poor validation dataset. ${R^2 = 0.83}, {\rho = 0.94}$. \textit{Right}: \METRICNAME{} over the poor validation dataset. $R^2=0.83, \rho=0.88$.}\label{fig:sim_metrics_baddata}
\end{figure}

\section{Plots of Q-value distributions}
\label{app:qvalue_distribution}

In Fig.~\ref{fig:q_dist}, we plot the Q-values of two real-world grasping models. The first is trained only in simulation and has poor real-world grasp success. The second is trained with a mix of simulated and real-world data. We plot a histogram of the average Q-value over each episode of validation set $\mathcal{D}$. The better model has a wider separation between successful Q-values and failed Q-values.

\begin{figure}[h]
    \centering
        \includegraphics[width=0.48\linewidth]{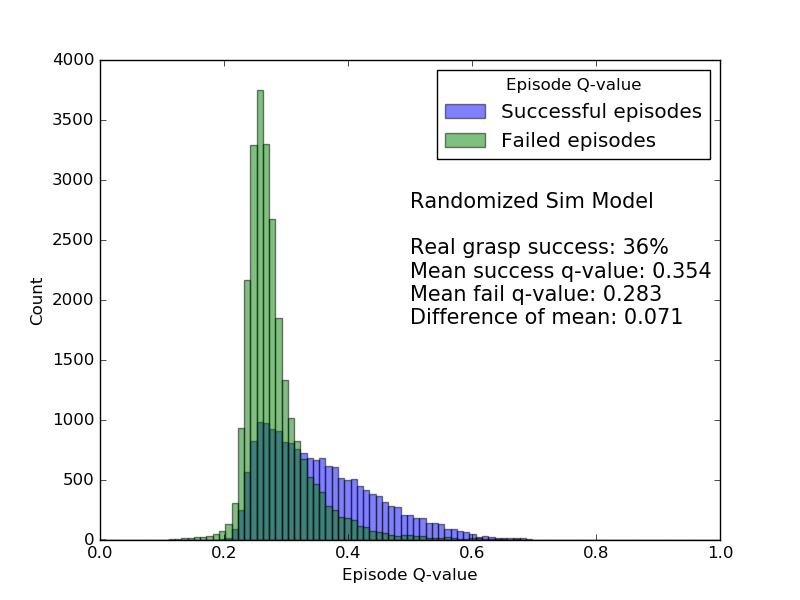} 
        \includegraphics[width=0.48\linewidth]{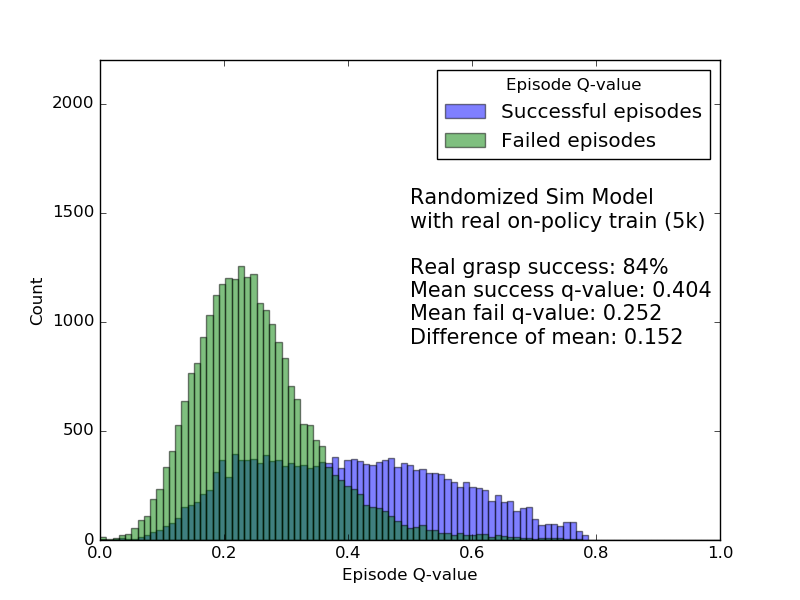}
    \caption{\textbf{Q-value distributions for successful and failed episodes.} \textit{Left:} Q-value distributions over successful and failed episodes in an off-policy data-set according to a learned policy with a poor grasp success rate of 36\%. \textit{Right:} The same distributions after the learned policy is improved by 5,000 grasps of real robot data, achieving a 84\% grasp success rate.}\label{fig:q_dist}
\end{figure}

\begin{figure}[h]
    \centering
    \begin{subfigure}[b]{0.48\linewidth}
        \includegraphics[width=\linewidth]{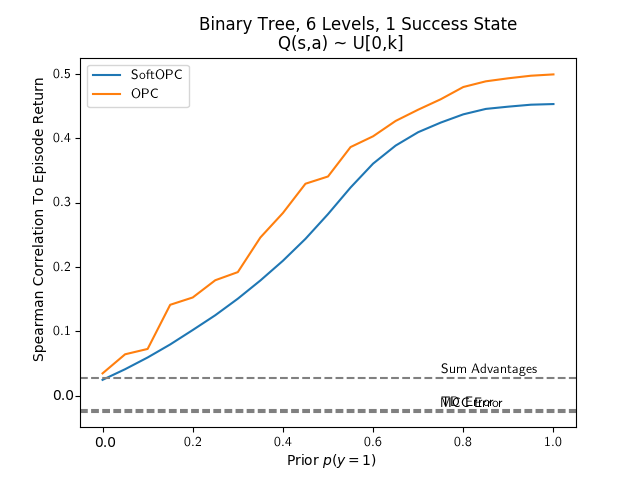}
        \caption{$Q(\bs,\ba)\sim U[0,k]$}
        \label{fig:magnitude_vary}
    \end{subfigure}
    \begin{subfigure}[b]{0.48\linewidth}
        \includegraphics[width=\linewidth]{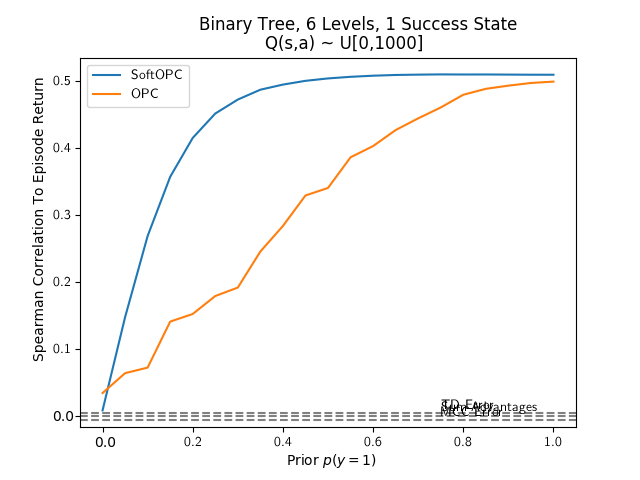}
        \caption{$Q(\bs,\ba)\sim U[0,1000]$}
        \label{fig:magnitude_big}
    \end{subfigure}
    \caption{Spearman correlation in binary tree with one success state for different Q-function generation methods. Varying magnitudes between Q-functions causes the \QDIFFNAME{} to perform worse.}\label{fig:magnitudes}
\end{figure}
\vspace{-0.10in}

\section{Comparison of \METRICNAME{} and \QDIFFNAME{}}
\label{app:opc_comparison}

We elaborate on the argument presented in the main paper, that \METRICNAME{} performs better when $Q(\bs,\ba)$ have different magnitudes, and otherwise \QDIFFNAME{} does better.
To do so, it is important to consider how the Q-functions were trained. In the tree environments, $Q(\bs,\ba)$ was sampled uniformly from $U[0,1]$, so $Q(\bs,\ba) \in [0,1]$ by construction. In the grasping environments, the network architecture ends in $\sigmoidfn(x)$, so $Q(\bs, \ba) \in [0,1]$. In these experiments, \QDIFFNAME{} did better. In Pong, $Q(\bs,\ba)$ was not constrained in any way, and these were the only experiments where discount factor $\gamma$ was varied between models. Here, \METRICNAME{} did better.

The hypothesis that Q-functions of varying magnitudes favor \METRICNAME{} can be validated in the tree environment. Again, we evaluate 1,000 Q-functions, but instead of sampling $Q(\bs,\ba)\sim U[0,1]$, the $k$th Q-function is sampled from $Q(\bs,\ba) \sim U[0,k]$. This produces 1,000 different magnitudes between the compared Q-functions.
Fig.~\ref{fig:magnitude_vary} demonstrates that when magnitudes are deliberately changed for each Q-function, the \QDIFFNAME{} performs worse, whereas the non-parametric \METRICNAME{} is unchanged. To demonstrate this is caused by a difference in magnitude, rather than large absolute magnitude, \METRICNAME{} and \QDIFFNAME{} are also evaluated over $Q(\bs,\ba)\sim U[0,1000]$. Every Q-function has high magnitude, but their magnitudes are consistently high. As seen in Fig.~\ref{fig:magnitude_big}, in this setting the \QDIFFNAME{} goes back to outperforming \METRICNAME{}.

\section{Scatterplots of each metric}

In Figure~\ref{fig:scatterplots}, we present scatterplots of each of the metrics in the simulated grasping environment from Sect.~\ref{sec:sim-robot}. We trained two Q-functions in a fully off-policy fashion, one with a dataset of $100,000$ episodes, and the other with a dataset of $900,000$ episodes. For every metric, we generate a scatterplot of all the model checkpoints. Each model checkpoint is color coded by whether it was trained with $100,000$ episodes or $900,000$ episodes.

\begin{figure}[h]
    \centering
    \includegraphics[width=0.45\linewidth]{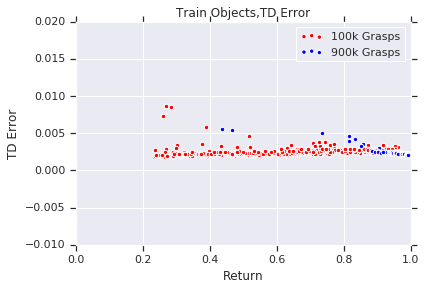}
    \includegraphics[width=0.45\linewidth]{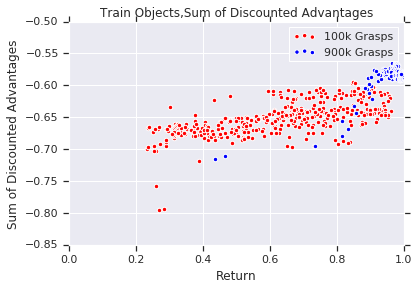}
    \includegraphics[width=0.45\linewidth]{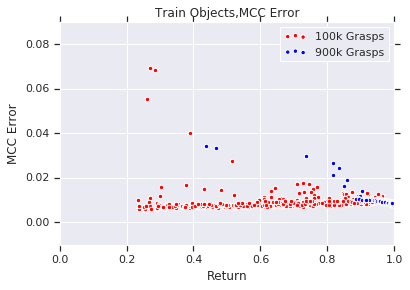}
    \includegraphics[width=0.45\linewidth]{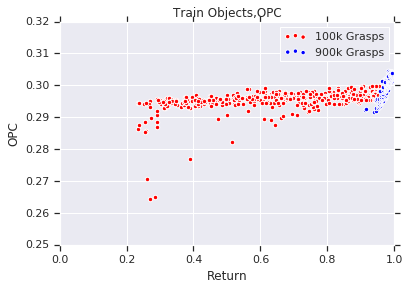}
    \includegraphics[width=0.45\linewidth]{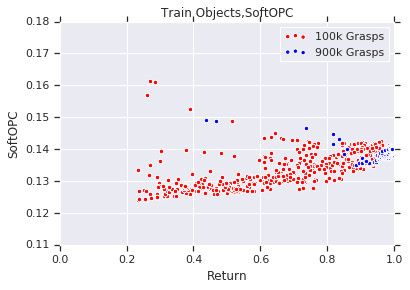}
    \caption{
    \textbf{Scatterplots of each metric in simulated grasping over train objects.} From left-to-right, top-to-bottom, we present scatterplots for: the TD error, $\sum \gamma^{t'} A^\pi(\bs_{t'}, \ba_{t'})$, MCC error, OPC, and SoftOPC.}
    \label{fig:scatterplots}
\end{figure}

\section{Extension to non-binary reward tasks}
\label{app:nonbinary}

Here, we present a template for how we could potentially extend \METRICNAME{} to non-binary, dense rewards.

First, we reduce dense reward MDPs to a sparse reward MDP, by applying a return-equivalent reduction introduced by~\citet{arjona2018rudder}. Given two MDPs, the two are \textit{return-equivalent} if (i) they differ only in the reward distribution and (ii) they have the same expected return at $t=0$ for every policy $\pi$. We show all dense reward MDPs can be reduced to a return-equivalent sparse reward MDP.

Given any MDP (${\cal S}, {\cal A}, {\cal P}, {\cal S}_0,r, \gamma$), we can augment the state space to create a return-equivalent MDP. First, augment state $\bs$ to be $(\bs, \br)$, where $\br$ is an added feature for accumulating reward. At the initial state, $\br = 0$. At time $t$ in the original MDP, whenever the policy would receive reward $r_t$, instead it receives no reward, but $r_t$ is added to the $\br$ feature. This gives $(\bs_1, 0), (\bs_2, r_1), (\bs_3, r_1+r_2)$, and so on. Adding $\br$ lets us maintain the Markov property, and at the final timestep, the policy receives reward equal to accumulated reward $\br$.

This MDP is return-equivalent and a sparse reward task. However, we do note that this requires adding an additional feature to the state space. Q-functions $Q(\bs,\ba)$ trained in the original MDP will not be aware of $\br$. To handle this, note that for $Q(\bs,\ba)$, the equivalent Q-function $Q'(\bs,\br,\ba)$ in the new MDP should satisfy $Q'(\bs,\br,\ba) = \br + Q(\bs, \ba)$, since $\br$ is return so far and $Q(\bs,\ba)$ is estimated future return from the original MDP. At evaluation time, we can compute $\br$ at each $t$ and adjust Q-values accordingly.

This reduction lets us consider just the sparse non-binary reward case. To do so, we first use a well-known lemma.

\begin{lemma}
Let $X$ be a discrete random variable over the real numbers with $n$ outcomes, where $X \in \{c_1, c_2, \cdots, c_n\}$, outcome $c_i$ occurs with probability $p_i$, and without loss of generalization $c_1 < c_2 < \cdots < c_n$. Then $\EX{}{X} = c_1 + \sum_{i=2}^{n} (c_i - c_{i-1})P(X \ge c_i)$.
\end{lemma}

\begin{proof}
Consider the right hand side. Substituting $1 = \sum_{i} p_i$ and expanding $P(X \ge c_i)$ gives
\begin{equation}
    c_1\sum_{j=1}^n p_j + \sum_{i=2}^n(c_i - c_{i-1})\sum_{j=i}^n p_j
\end{equation}
For each $p_j$, the coefficient is $c_1 + (c_2-c_1) + \cdots + (c_j - c_{j-1}) = c_j$. The total sum is $\sum_{j=1}^n c_jp_j$, which matches the expectation $\EX{}{X}$.
\end{proof}

The return $R(\pi)$ can be viewed as a random variable depending on $\pi$ and the environment. We add a simplifying assumption: the environment has a finite number of return outcomes, all of which appear in dataset $\mathcal{D}$. Letting those outcomes be $c_1, c_2, \cdots, c_n$, estimating $P(R(\pi) \ge c_i)$ for each $c_i$ would be sufficient to estimate expected return.

The key insight is that to estimate $P(R(\pi) \ge c_i)$, we can define one more binary reward MDP. In this MDP, reward is sparse, and the final return is $1$ if original return is $\ge c_i$, and $0$ otherwise. The return of $\pi$ in this new MDP is exactly $P(R(\pi) \ge c_i)$, and can be estimated with \METRICNAME{}. The final pseudocode is provided below.

\begin{algorithm}
	\caption{Thresholded Off-Policy Classification} 
	\begin{algorithmic}[1]
	    \Require Dataset $\mathcal{D}$ of trajectories $\tau=(\bs_1, \ba_1, \ldots, \bs_T, \ba_T)$, learned Q-function $Q(\bs, \ba)$, return threshold $c_i$
	    \For {$\tau \in \mathcal{D}$} \Comment{Generate Q-values and check threshold}
	        \For {$t = 1, 2, \ldots, T$}
	            \State $\br_t \gets \sum_{t'=1}^t r_{t'}$
	            \State Compute $Q'(\bs_t, \br, \ba_t) = \br_t + Q(\bs, \ba)$
	        \EndFor
	        \State Save each $Q'(\bs_t, \br, \ba_t)$, as well as whether $\sum_{t=1}^T r_t$ exceeds $c_i$.
	    \EndFor
	    \State \Return OPC for the computed $Q'(\bs, \br, \ba)$ and threshold $c_i$.
	\end{algorithmic}
	\label{alg:threshold-opc}
\end{algorithm}
\begin{algorithm}
	\caption{Extended Off-Policy Classification} 
	\begin{algorithmic}[1]
	    \Require Dataset $\mathcal{D}$ of trajectories $\tau=(\bs_1, \ba_1, \ldots, \bs_T, \ba_T)$, learned Q-function $Q(\bs, \ba)$.
	    \State $c_1, \cdots, c_n \gets$ The $n$ distinct return outcomes within $\mathcal{D}$, sorted
	    \State total $\gets c_1$
	    \For {$i = 2, \ldots, n$}
	        \State total $\mathrel{+}= (c_i - c_{i-1})$ThresholdedOPC($Q(\bs, \ba), \mathcal{D}, c_i$)
	    \EndFor
		\State \Return total
	\end{algorithmic}
	\label{alg:extended-opc}
\end{algorithm}

%% file: main_neurips.bbl
\begin{thebibliography}{40}
\providecommand{\natexlab}[1]{#1}
\providecommand{\url}[1]{\texttt{#1}}
\expandafter\ifx\csname urlstyle\endcsname\relax
  \providecommand{\doi}[1]{doi: #1}\else
  \providecommand{\doi}{doi: \begingroup \urlstyle{rm}\Url}\fi

\bibitem[Arjona-Medina et~al.(2018)Arjona-Medina, Gillhofer, Widrich,
  Unterthiner, Brandstetter, and Hochreiter]{arjona2018rudder}
Arjona-Medina, J.~A., Gillhofer, M., Widrich, M., Unterthiner, T.,
  Brandstetter, J., and Hochreiter, S.
\newblock Rudder: Return decomposition for delayed rewards.
\newblock \emph{arXiv preprint arXiv:1806.07857}, 2018.

\bibitem[Babaeizadeh et~al.(2018)Babaeizadeh, Finn, Erhan, Campbell, and
  Levine]{babaeizadeh2018stochastic}
Babaeizadeh, M., Finn, C., Erhan, D., Campbell, R.~H., and Levine, S.
\newblock Stochastic variational video prediction.
\newblock In \emph{International Conference on Representation Learning}, 2018.

\bibitem[Bellemare et~al.(2013)Bellemare, Naddaf, Veness, and
  Bowling]{bellemare2013arcade}
Bellemare, M.~G., Naddaf, Y., Veness, J., and Bowling, M.
\newblock The arcade learning environment: An evaluation platform for general
  agents.
\newblock \emph{Journal of Artificial Intelligence Research}, 47:\penalty0
  253--279, 2013.

\bibitem[Brockman et~al.(2016)Brockman, Cheung, Pettersson, Schneider,
  Schulman, Tang, and Zaremba]{brockman2016openai}
Brockman, G., Cheung, V., Pettersson, L., Schneider, J., Schulman, J., Tang,
  J., and Zaremba, W.
\newblock Openai gym.
\newblock \emph{arXiv preprint arXiv:1606.01540}, 2016.

\bibitem[Cobbe et~al.(2018)Cobbe, Klimov, Hesse, Kim, and
  Schulman]{cobbe2018quantifying}
Cobbe, K., Klimov, O., Hesse, C., Kim, T., and Schulman, J.
\newblock Quantifying generalization in reinforcement learning.
\newblock \emph{arXiv preprint arXiv:1812.02341}, 2018.

\bibitem[Deng et~al.(2009)Deng, Dong, Socher, Li, Li, and
  Fei-Fei]{imagenet_cvpr09}
Deng, J., Dong, W., Socher, R., Li, L.-J., Li, K., and Fei-Fei, L.
\newblock {ImageNet: A Large-Scale Hierarchical Image Database}.
\newblock In \emph{CVPR, 2009}, 2009.

\bibitem[Dudik et~al.(2011)Dudik, Langford, and Li]{Dudik2011-doublyrobust}
Dudik, M., Langford, J., and Li, L.
\newblock Doubly robust policy evaluation and learning.
\newblock In \emph{ICML}, March 2011.

\bibitem[Dud{\'\i}k et~al.(2014)Dud{\'\i}k, Erhan, Langford, Li,
  et~al.]{dudik2014doubly}
Dud{\'\i}k, M., Erhan, D., Langford, J., Li, L., et~al.
\newblock Doubly robust policy evaluation and optimization.
\newblock \emph{Statistical Science}, 29\penalty0 (4):\penalty0 485--511, 2014.

\bibitem[Farahmand \& Szepesv{\'a}ri(2011)Farahmand and
  Szepesv{\'a}ri]{Farahmand2011-ec}
Farahmand, A.-M. and Szepesv{\'a}ri, C.
\newblock Model selection in reinforcement learning.
\newblock \emph{Mach. Learn.}, 85\penalty0 (3):\penalty0 299--332, December
  2011.

\bibitem[Hanna et~al.(2017)Hanna, Stone, and Niekum]{Hanna2017-yr}
Hanna, J.~P., Stone, P., and Niekum, S.
\newblock Bootstrapping with models: Confidence intervals for {Off-Policy}
  evaluation.
\newblock In \emph{Proceedings of the 16th Conference on Autonomous Agents and
  {MultiAgent} Systems}, AAMAS '17, pp.\  538--546, Richland, SC, 2017.
  International Foundation for Autonomous Agents and Multiagent Systems.

\bibitem[Horvitz \& Thompson(1952)Horvitz and
  Thompson]{horvitz1952generalization}
Horvitz, D.~G. and Thompson, D.~J.
\newblock A generalization of sampling without replacement from a finite
  universe.
\newblock \emph{Journal of the American statistical Association}, 47\penalty0
  (260):\penalty0 663--685, 1952.

\bibitem[James et~al.(2019)James, Wohlhart, Kalakrishnan, Kalashnikov, Irpan,
  Ibarz, Levine, Hadsell, and Bousmalis]{james2019rcan}
James, S., Wohlhart, P., Kalakrishnan, M., Kalashnikov, D., Irpan, A., Ibarz,
  J., Levine, S., Hadsell, R., and Bousmalis, K.
\newblock Sim-to-real via sim-to-sim: Data-efficient robotic grasping via
  randomized-to-canonical adaptation networks.
\newblock In \emph{IEEE Conference on Computer Vision and Pattern Recognition},
  March 2019.

\bibitem[Jiang \& Li(2015)Jiang and Li]{Jiang2015-doublyrobust}
Jiang, N. and Li, L.
\newblock Doubly robust off-policy value evaluation for reinforcement learning.
\newblock November 2015.

\bibitem[Kakade \& Langford(2002)Kakade and Langford]{kakade2002}
Kakade, S. and Langford, J.
\newblock Approximately optimal approximate reinforcement learning.
\newblock In \emph{ICML}, 2002.

\bibitem[Kalashnikov et~al.(2018)Kalashnikov, Irpan, Pastor, Ibarz, Herzog,
  Jang, Quillen, Holly, Kalakrishnan, Vanhoucke, et~al.]{kalashnikov2018qt}
Kalashnikov, D., Irpan, A., Pastor, P., Ibarz, J., Herzog, A., Jang, E.,
  Quillen, D., Holly, E., Kalakrishnan, M., Vanhoucke, V., et~al.
\newblock Qt-opt: Scalable deep reinforcement learning for vision-based robotic
  manipulation.
\newblock \emph{arXiv preprint arXiv:1806.10293}, 2018.

\bibitem[Kiryo et~al.(2017)Kiryo, Niu, du~Plessis, and
  Sugiyama]{kiryo2017positive}
Kiryo, R., Niu, G., du~Plessis, M.~C., and Sugiyama, M.
\newblock Positive-unlabeled learning with non-negative risk estimator.
\newblock In \emph{NeurIPS}, pp.\  1675--1685, 2017.

\bibitem[Koos et~al.(2010)Koos, Mouret, and Doncieux]{koos2010crossing}
Koos, S., Mouret, J.-B., and Doncieux, S.
\newblock Crossing the reality gap in evolutionary robotics by promoting
  transferable controllers.
\newblock In \emph{Proceedings of the 12th annual conference on Genetic and
  evolutionary computation}, pp.\  119--126. ACM, 2010.

\bibitem[Koos et~al.(2012)Koos, Mouret, and Doncieux]{koos2012transferability}
Koos, S., Mouret, J.-B., and Doncieux, S.
\newblock The transferability approach: Crossing the reality gap in
  evolutionary robotics.
\newblock \emph{IEEE Transactions on Evolutionary Computation}, 17\penalty0
  (1):\penalty0 122--145, 2012.

\bibitem[Lee et~al.(2018)Lee, Zhang, Ebert, Abbeel, Finn, and
  Levine]{lee2018stochastic}
Lee, A.~X., Zhang, R., Ebert, F., Abbeel, P., Finn, C., and Levine, S.
\newblock Stochastic adversarial video prediction.
\newblock \emph{arXiv preprint arXiv:1804.01523}, 2018.

\bibitem[Lillicrap et~al.(2015)Lillicrap, Hunt, Pritzel, Heess, Erez, Tassa,
  Silver, and Wierstra]{lillicrap2015continuous}
Lillicrap, T.~P., Hunt, J.~J., Pritzel, A., Heess, N., Erez, T., Tassa, Y.,
  Silver, D., and Wierstra, D.
\newblock Continuous control with deep reinforcement learning.
\newblock \emph{arXiv preprint arXiv:1509.02971}, 2015.

\bibitem[Liu et~al.(2018)Liu, Gottesman, Raghu, Komorowski, Faisal,
  Doshi-Velez, and Brunskill]{liu2018representation}
Liu, Y., Gottesman, O., Raghu, A., Komorowski, M., Faisal, A., Doshi-Velez, F.,
  and Brunskill, E.
\newblock Representation balancing mdps for off-policy policy evaluation.
\newblock In \emph{NeurIPS}, 2018.

\bibitem[Machado et~al.(2018)Machado, Bellemare, Talvitie, Veness, Hausknecht,
  and Bowling]{machado2018revisiting}
Machado, M.~C., Bellemare, M.~G., Talvitie, E., Veness, J., Hausknecht, M., and
  Bowling, M.
\newblock Revisiting the arcade learning environment: Evaluation protocols and
  open problems for general agents.
\newblock \emph{Journal of Artificial Intelligence Research}, 61:\penalty0
  523--562, 2018.

\bibitem[Mahmood et~al.(2014)Mahmood, van Hasselt, and
  Sutton]{mahmood2014weighted}
Mahmood, A.~R., van Hasselt, H.~P., and Sutton, R.~S.
\newblock Weighted importance sampling for off-policy learning with linear
  function approximation.
\newblock In \emph{Advances in Neural Information Processing Systems}, pp.\
  3014--3022, 2014.

\bibitem[Mannor et~al.(2007)Mannor, Simester, Sun, and
  Tsitsiklis]{mannor2007bias}
Mannor, S., Simester, D., Sun, P., and Tsitsiklis, J.~N.
\newblock Bias and variance approximation in value function estimates.
\newblock \emph{Management Science}, 53\penalty0 (2):\penalty0 308--322, 2007.

\bibitem[Mnih et~al.(2015)Mnih, Kavukcuoglu, Silver, Rusu, Veness, Bellemare,
  Graves, Riedmiller, Fidjeland, Ostrovski, et~al.]{mnih2015human}
Mnih, V., Kavukcuoglu, K., Silver, D., Rusu, A.~A., Veness, J., Bellemare,
  M.~G., Graves, A., Riedmiller, M., Fidjeland, A.~K., Ostrovski, G., et~al.
\newblock Human-level control through deep reinforcement learning.
\newblock \emph{Nature}, 518\penalty0 (7540):\penalty0 529, 2015.

\bibitem[Murphy(2005)]{Murphy2005-qc}
Murphy, S.~A.
\newblock A generalization error for {Q-Learning}.
\newblock \emph{J. Mach. Learn. Res.}, 6:\penalty0 1073--1097, July 2005.

\bibitem[Nichol et~al.(2018)Nichol, Pfau, Hesse, Klimov, and
  Schulman]{nichol2018gotta}
Nichol, A., Pfau, V., Hesse, C., Klimov, O., and Schulman, J.
\newblock Gotta learn fast: A new benchmark for generalization in rl.
\newblock \emph{arXiv preprint arXiv:1804.03720}, 2018.

\bibitem[Precup et~al.(2000)Precup, Sutton, and Singh]{precup2000eligibility}
Precup, D., Sutton, R.~S., and Singh, S.
\newblock Eligibility traces for off-policy policy evaluation.
\newblock In \emph{Proceedings of the Seventeenth International Conference on
  Machine Learning, 2000}, pp.\  759--766. Morgan Kaufmann, 2000.

\bibitem[Quillen et~al.(2018)Quillen, Jang, Nachum, Finn, Ibarz, and
  Levine]{Quillen2018-ms}
Quillen, D., Jang, E., Nachum, O., Finn, C., Ibarz, J., and Levine, S.
\newblock Deep reinforcement learning for {Vision-Based} robotic grasping: A
  simulated comparative evaluation of {Off-Policy} methods.
\newblock February 2018.

\bibitem[Raghu et~al.(2018)Raghu, Irpan, Andreas, Kleinberg, Le, and
  Kleinberg]{raghu2018can}
Raghu, M., Irpan, A., Andreas, J., Kleinberg, R., Le, Q., and Kleinberg, J.
\newblock Can deep reinforcement learning solve erdos-selfridge-spencer games?
\newblock In \emph{International Conference on Machine Learning}, pp.\
  4235--4243, 2018.

\bibitem[Riedmiller et~al.(2018)Riedmiller, Hafner, Lampe, Neunert, Degrave,
  Van~de Wiele, Mnih, Heess, and Springenberg]{riedmiller2018learning}
Riedmiller, M., Hafner, R., Lampe, T., Neunert, M., Degrave, J., Van~de Wiele,
  T., Mnih, V., Heess, N., and Springenberg, J.~T.
\newblock Learning by playing-solving sparse reward tasks from scratch.
\newblock In \emph{International Conference on Machine Learning}, 2018.

\bibitem[Ross \& Bagnell(2010)Ross and Bagnell]{ross2010efficient}
Ross, S. and Bagnell, D.
\newblock Efficient reductions for imitation learning.
\newblock In \emph{AISTATS}, pp.\  661--668, 2010.

\bibitem[S.~Spearman(1904)]{spearman1904}
S.~Spearman, C.
\newblock The proof and measurement of association between two things.
\newblock \emph{The American Journal of Psychology}, 15:\penalty0 72--101, 01
  1904.
\newblock \doi{10.2307/1412159}.

\bibitem[Theocharous et~al.(2015)Theocharous, Thomas, and
  Ghavamzadeh]{theocharous2015personalized}
Theocharous, G., Thomas, P.~S., and Ghavamzadeh, M.
\newblock Personalized ad recommendation systems for life-time value
  optimization with guarantees.
\newblock In \emph{IJCAI}, pp.\  1806--1812, 2015.

\bibitem[Thomas \& Brunskill(2016)Thomas and Brunskill]{Thomas2016-zn}
Thomas, P. and Brunskill, E.
\newblock {Data-Efficient} {Off-Policy} policy evaluation for reinforcement
  learning.
\newblock In \emph{International Conference on Machine Learning}, pp.\
  2139--2148, June 2016.

\bibitem[Thomas et~al.(2015)Thomas, Theocharous, and
  Ghavamzadeh]{Thomas2015-highconfidence}
Thomas, P.~S., Theocharous, G., and Ghavamzadeh, M.
\newblock {High-Confidence} {Off-Policy} evaluation.
\newblock \emph{AAAI}, 2015.

\bibitem[Todorov et~al.(2012)Todorov, Erez, and Tassa]{todorov2012mujoco}
Todorov, E., Erez, T., and Tassa, Y.
\newblock Mujoco: A physics engine for model-based control.
\newblock In \emph{Intelligent Robots and Systems (IROS), 2012 IEEE/RSJ
  International Conference on}, pp.\  5026--5033. IEEE, 2012.

\bibitem[van Hasselt et~al.(2016)van Hasselt, Guez, and Silver]{van2016deep}
van Hasselt, H., Guez, A., and Silver, D.
\newblock Deep reinforcement learning with double q-learning.
\newblock In \emph{Thirtieth AAAI Conference on Artificial Intelligence}, 2016.

\bibitem[Zhang et~al.(2018{\natexlab{a}})Zhang, Ballas, and
  Pineau]{Zhang2018-ah}
Zhang, A., Ballas, N., and Pineau, J.
\newblock A dissection of overfitting and generalization in continuous
  reinforcement learning.
\newblock June 2018{\natexlab{a}}.

\bibitem[Zhang et~al.(2018{\natexlab{b}})Zhang, Vinyals, Munos, and
  Bengio]{Zhang2018-qj}
Zhang, C., Vinyals, O., Munos, R., and Bengio, S.
\newblock A study on overfitting in deep reinforcement learning.
\newblock April 2018{\natexlab{b}}.

\end{thebibliography}
